\documentclass[11pt,letter,citesort]{article}

\usepackage{booktabs}

\usepackage[round,colon,authoryear]{natbib}

\usepackage[utf8]{inputenc}
\usepackage{amsmath}
\usepackage{amssymb}
\usepackage{amsfonts}
\usepackage{amscd}
\usepackage{latexsym}
\usepackage{wasysym}
\usepackage{graphicx}
\usepackage[table,pdftex,hyperref,svgnames]{xcolor}
\usepackage{hyperref,doi}
\usepackage[all,cmtip]{xy}
\usepackage{bbold}
\usepackage{mathtools}

\usepackage{xspace}

\usepackage{multirow}

\usepackage{amsfonts}
\usepackage{amssymb}
\usepackage{amsthm}
\usepackage{enumitem}
\usepackage{mathtools,bbm}
\usepackage{xpatch,multicol,array}

\usepackage[margin=0.8in]{geometry}

\usepackage{lipsum}
\usepackage{mathtools}
\usepackage{cuted}
\usepackage{mathtools,bbm}

\usepackage[caption=false,font=footnotesize]{subfig}
\graphicspath{{./images/}{./fig}}

\setlist{nosep}

\newcommand{\R}{\mathbb{R}}

\newtheorem{theorem}{Theorem}[section]
\newtheorem{corollary}[theorem]{Corollary}
\newtheorem{lemma}[theorem]{Lemma}
\newtheorem{proposition}[theorem]{Proposition}

{
      \theoremstyle{plain}
            
}

\newtheorem{remark}[theorem]{Remark} 

\DeclareSymbolFont{bbold}{U}{bbold}{m}{n}
\DeclareSymbolFontAlphabet{\mathbbold}{bbold}
\newcommand{\vect}[1]{\mathbbold{#1}}

\newcommand{\trace}{\operatorname{trace}}

\newcommand{\setdef}[2]{\left\{#1 \; | \; #2\right\}}

\usepackage{algpseudocode,algorithm,algorithmicx}
\usepackage{algpseudocode}

\newcommand\norm[1]{\left\lVert#1\right\rVert}

\newcommand{\subscr}[2]{#1_{\textup{#2}}}

\newcommand{\Sd}{\subscr{\mathbb{S}}{d}}

\newcommand{\Sdpp}{\subscr{\mathbb{S}}{d}^{\operatorname{++}}}
\newcommand{\vecc}{\textbf{\textup{vec}}}

\def\inv{^{-1}}

\begin{document}

\title{Distributionally Robust Formulation and Model Selection for the Graphical Lasso}

\author{Pedro Cisneros-Velarde$^1$, Sang-Yun Oh$^2$, Alexander Petersen$^2$\\
Center for Control, Dynamical Systems and Computation$^1$, \\
Department of Statistics and Applied Probability$^2$,\\
University of California, Santa Barbara}
 \date{}
\maketitle

\begin{abstract}
Building on a recent framework for distributionally robust optimization, we 
consider estimation of the inverse covariance matrix for multivariate data. We provide a novel notion of a Wasserstein ambiguity set specifically tailored to this estimation problem, 
leading to a tractable class of regularized estimators. Special cases include penalized likelihood estimators for Gaussian data, specifically the graphical lasso estimator. As a consequence of this formulation, the radius of the Wasserstein ambiguity set is directly related to the regularization parameter in the estimation problem.  Using this relationship, 
the level of robustness of the estimation procedure can be shown to correspond to the level of confidence with which the ambiguity set contains a distribution with the population covariance. Furthermore, a unique feature of our formulation is that the radius can be expressed in closed-form as a function of the ordinary sample covariance matrix. Taking advantage of this finding, we develop a simple algorithm to determine a regularization parameter for graphical lasso, using only the bootstrapped sample covariance matrices, meaning that computationally expensive repeated evaluation of the graphical lasso algorithm is not necessary. Alternatively, the distributionally robust formulation can also quantify the robustness of the corresponding estimator if one uses an off-the-shelf method such as cross-validation. Finally, we numerically study the obtained regularization criterion and analyze the robustness of other automated tuning procedures used in practice.
\end{abstract}

\section{Introduction}

In statistics and machine learning, the covariance matrix $\Sigma$ of a random vector $X \in \R^d$ is a fundamental quantity for characterizing marginal pairwise dependencies between variables. Furthermore, the inverse covariance matrix $\Omega = \Sigma^{-1}$ provides information about the \emph{conditional} linear dependency structure between the variables.  For example, in the case that $X$ is Gaussian, $\Omega_{jk} = 0$ if and only if the $j$th and $k$th variables of $X$ are conditionally independent given the rest. Such relationships are of interest in many applications such as environmental science, biology, and neuroscience \citep{guillot2015,HUANG2010935,Krumsiek2011}, and have given rise to various statistical and machine learning methods for inverse covariance estimation. 


Given an independent sample $X_i \sim X,$ $i = 1,\ldots,n,$ the sample covariance $A_n = \frac{1}{n}\sum^n_{i=1}X_iX_i^\top$, which is the maximum likelihood estimator if $X$ is Gaussian, can be a poor estimator of $\Sigma$ unless $d/n$ is very small. Driven by a so-called high dimensional setting where $n\ll d$, where $A_n$ is not invertible, regularized estimation of the precision matrix has gained significant interest~\citep{Cai2011,DE:00,JF-TH-RT:07,Khare2015,Won2012,MY-LY:07}. Such regularization procedures are useful even when $A_n$ is a stable estimate (i.e., positive definite with small condition number), since the inverse covariance estimate $A_n^{-1}$ is dense and will not reflect the sparsity of corresponding nonzero elements in $\Omega$. 

\textbf{Distributionally Robust Optimization}:
Let $\Sd$ be the set of $d\times{d}$ symmetric matrices, and 
$\Sdpp \subset \Sd$ be the subset 
of positive definite symmetric matrices. 
Given a loss function $l(X;K)$ for $K \in \Sdpp$ and $X \in \mathbb{R}^d$, a classical approach would be to estimate $\Omega$ by minimizing the empirical loss $\sum_{i = 1}^n l(X_i;K)$ over $K \in \Sdpp,$ perhaps including a regularization or penalty term.  The error in this estimate arises from the discrepancy between the true data-generating distribution and the observed training samples, and can be assessed by various tools such as concentration bounds or rates of convergence.  In contrast, \emph{distributionally robust optimization} (DRO) is a technique that explicitly incorporates uncertainty about the distribution of the $X_i$ into the estimation procedure.
For an introduction on the general topic of DRO, we refer to the works~\citep{SS-PME-DK:15,JB-KM:16,SSA19} and the references therein. In the context of inverse covariance estimation, a distributionally robust estimate of $\Omega$ is obtained by solving
\begin{equation}
\label{for_now}
\inf_{K\in\Sdpp}\sup_{P\in\mathcal{S}}E_P[l(K;X)],
\end{equation}
where $\mathcal{S}$, known as an \emph{ambiguity set}, is a collection of probability measures on $\mathbb{R}^d$. As pointed out in~\citep{JB-KM:16}, a natural choice for $\mathcal{S}$ is the neighborhood $\setdef{P}{D(P,\mu)\leq \delta}$, with $\mu$ being a chosen baseline model, $\delta$ being some tolerance level which defines the uncertainty size of the ambiguity set, and $D$ being some discrepancy metric between two probability measures. In a practical setting, we have access to some samples (or data points) from the unknown distribution, and thus, a good candidate for the baseline model is the empirical measure.  

Very recent work by~\cite{VAN-DK-PME:18}, also analyzed by~\cite{JB-NS:19}, used the DRO framework to construct a new regularized (dense) inverse covariance estimator.  Working under the assumption that $X$ is Gaussian, the authors construct an ambiguity set $\mathcal{S}$ of Gaussian distributions that, up to a certain tolerance level, are consistent with the observed data. Recent work on DRO in other machine learning problems has revealed explicit connections to well-known regularized estimators, specifically regularized logistic regression~\citep{SS-PME-DK:15} and the square-root lasso for linear models~\citep{JB-YK-KM:16}; however, such a connection to regularized \emph{sparse} inverse covariance estimators that are used in practice has yet to be made. 

\textbf{The Graphical Lasso}:   One of the most common methods to recover the sparsity pattern in $\Omega$ is to add an $l_1$-regularization term to the Gaussian likelihood function, motivated by the consideration of the \emph{Gaussian Graphical Model} (GGM). A sparse estimate of $\Omega = \Sigma\inv$ is produced by minimizing
\begin{equation}
\label{zero2}
\mathcal{L}_{\lambda}(K) = \frac{1}{n}\sum^{n}_{i=1}X_i^\top K X_i-\log|K|+\lambda\sum^d_{i=1}\sum_{j=1}^d|k_{ij}|,
\end{equation}
where $k_{ij}$ is the $(i,j)$ entry of $K$ and $\lambda > 0$ is a user-specified regularization parameter~\citep{OB-LEG-AA:08,JF-TH-RT:07,MY-LY:07}.  Although several algorithms exist to solve this objective function~\citep{JF-TH-RT:07,Rolfs2012,Hsieh14a}, the minimizer of \eqref{zero2} is often referred to as \emph{graphical lasso} estimator~\citep{JF-TH-RT:07}.  The first two terms of \eqref{zero2} are related to Stein's loss~\citep{WJ-CS:61} when evaluated at the empirical measure, and also correspond to the negative log-likelihood up to an additive constant if $X$ is Gaussian.  The performance of the graphical lasso estimator in high-dimensional settings has been investigated \citep{AJR-PJB-EL-JZ:08,JJ-SvdG:18}, 
as well as modifications and extensions that implement some notion of \emph{robustness}, i.e., for making it robust to outliers or relaxing the normality assumptions in the data \citep{Khare2015,CL-JF:09,PLH-XLT:18,LX-HZ:12,EY-ACL:15}.

Besides its theoretical relevance, the graphical lasso and its extensions also enjoy many practical advantages. For example, it has been used as a network inference tool. In these applications, the precision matrix can indicate 
which nodes in a network are conditionally independent given information from remaining nodes, thus giving an indication of network functionality. This has been important in neuroscience applications when studying the inference of brain connectivity~\citep{SY-QS-SJ-PW-ID-JY:15,SS-KM-GS-MW-CB-TN-JR-MW:15,HUANG2010935}. Applications in gene regulatory networks and metabolomics have also been reported~\citep{PM-AY-OOO:10,NS-SK-HK-OOO:18,Krumsiek2011}.

The performance of the graphical lasso estimator hinges critically on the choice of $\lambda.$ While there have been studies on how to properly tune $\lambda$ to obtain a consistent estimator or to establish correct detection of nonzero elements in the precision matrix~\citep{AJR-PJB-EL-JZ:08,OB-LEG-AA:08,RM-TH:12},  
%
in practice, this selection 
is often made through automated methods like cross-validation. 
%
%



\textbf{Contributions}:  In this paper, we propose a distributionally robust reformulation of the graphical lasso estimator in \eqref{zero2}.  Following~\cite{SS-PME-DK:15,JB-YK-KM:16,PME-DK:18}, we utilize the Wasserstein metric to quantify distributional uncertainty for the construction of the ambiguity set. The following points summarize our main contributions.  

\begin{itemize}
\item We formulate a class of DRO problems for inverse covariance estimation, leading to a tractable class of $\ell_p$-norm regularized estimators.  As the graphical lasso estimator~\eqref{zero2} is a special case, this provides us with a new interpretation of this popular technique.  This DRO formulation is made possible by a novel type of ambiguity set, now defined as a collection of measures on matrices.  This nontrivial adaptation is necessary due to the fact that a direct generalization of other DRO approaches using vector-valued data (e.g., \cite{SSA19}) does not result in a closed-form regularization problem, and thus does not provide the desired connection to the graphical lasso. 
\item We use this formulation to suggest a criterion for the selection of the regularization parameter in the estimation problem in the classical regime $n>d$. This criterion follows the \emph{Robust Wasserstein Profile} (RWP) inference recently introduced by~\cite{JB-YK-KM:16}, which makes no assumption on the normality of the data, and which we tailor to our specific problem.  The proposed criterion expresses the regularization parameter as an explicit function of the sample covariance $A_n,$ unlike other instances where RWP has been implemented which rely on stochastic dominance arguments.

\item We formulate a novel \emph{robust selection} (RobSel) algorithm for regularization parameter choice.  Focusing on the graphical lasso, we provide numerical results that compare the performance of cross-validation and our proposed algorithm for the selection of the regularization term. 
\end{itemize}


%
%

The paper is organized as follows. In Section 2 we describe our main theoretical result: the distributionally robust formulation of the regularized inverse covariance (log-likelihood) estimation, from which graphical lasso is a particular instance. In Section 3 we propose a criterion for choosing the regularization parameter inspired by this formulation and outline the bootstrap-based RobSel algorithm for its computation. In Section 4 we present some numerical results comparing the proposed criterion of Section 3 with cross-validation. Finally, we state some concluding remarks and future research directions in Section 5. All proofs of theoretical results can be found in the supplementary material.

\section{A Distributionally Robust Formulation of the Graphical lasso}
\label{s_3}

First, we provide preliminary details on notation. Given a matrix $A\in\R^{d\times{d}}$, $a_{jk}$ denotes its $(j,k)$ entry and $\vecc(A)\in\R^{d^2}$ denotes its vectorized form, which we assume to be in a row major fashion. For matrices denoted by Greek letters, its entries are simply denoted by appropriate subscripts, i.e.\ $\Sigma_{jk}$.  The operator $|\cdot|$, when applied to a matrix, denotes its determinant; when applied to a scalar or a vector, it denotes the absolute value or entry-wise absolute value, respectively. 
The $\ell_p$-norm of a vector is denoted by $\norm{\cdot}_p$.  We use the symbol $\Rightarrow$ to denote convergence in distribution.

Recall that $X\in\R^{d}$ is a zero-mean random vector with covariance matrix $\Sigma\in\Sdpp$. Let $\mathbb{Q}_0$ be the probability law for $X$ and $\Omega=\Sigma^{-1}$ be the precision matrix. Define the \textit{graphical loss function} as 
\begin{align}
\label{onee1}
\begin{split}
l(X;K)&=X^\top K X-\log|K|\\
&=\trace(KXX^\top)-\log|K|.
\end{split}
\end{align}
Then $E_{\mathbb{Q}_0}[l(K;X)]=\trace(K\Sigma)-\log|K|$ is a convex function of $K$ over the convex cone $\Sdpp$.  Using the first-order optimality criterion, we observe that $K=\Omega$ sets the gradient $\frac{\partial}{\partial K}E_{\mathbb{Q}_0}[l(X,K)]=\Sigma - K^{-1}$ equal to the zero matrix (see~\cite[Appendix A]{SB-LV:04} for details on this differentiation). Hence,
\begin{equation*}
\arg\min_{K\in\Sdpp}E_{\mathbb{Q}_0}[l(X;K)]=\Omega.
\end{equation*}  
so that \eqref{onee1} is a consistent loss function.  

Now, if we consider an iid random sample $X_1,\cdots,X_n\sim X$, $n > d,$ with empirical measure $\mathbb{Q}_n,$ then
\begin{align*}
\arg\min_{K\in\Sdpp}E_{\mathbb{Q}_n}[l(X;K)] &= \arg\min_{K\in\Sdpp} \frac{1}{n}\sum_{i=1}^{n}l(X_i;K)\\
&=A_n^{-1}
\end{align*}  
with $A_n=\frac{1}{n}\sum^{n}_{i=1}X_iX_i^\top$. Thus, as described in the Introduction, a natural approach would be to implement the DRO procedure outlined in \cite{PME-DK:18} by building an ambiguity set based on perturbations of $\mathbb{Q}_n,$ leading to the DRO estimate given by \eqref{for_now}.  However, this approach does not convert \eqref{for_now} into a regularized estimation problem as desired, since the inner supremum cannot be explicitly given in closed-form. For more details, see section A in the supplementary material.

As an alternative, let $\mathbb{P}_0$ represent the measure of the random matrix $W=XX^\top$ on $\Sd$ induced by $\mathbb{Q}_0$ and, similarly let $\mathbb{P}_n$ be empirical measure of the sample $W_i = X_iX_i^\top$, $i = 1,\ldots,n.$  Redefining the graphical loss function $l: \Sd \times \Sdpp$ as 
\begin{align}
\label{onee11}
\begin{split}
l(W;K)&=\trace(KW)-\log|K|,
\end{split}
\end{align}
then 
\begin{align*}
\Omega &= \arg\min_{K \in \Sdpp}E_{\mathbb{P}_0}[l(W;K)]\\
A_n\inv &= \arg\min_{K \in \Sdpp}E_{\mathbb{P}_n}[l(W;K)].    
\end{align*}  
This observation leads to a tractable DRO formulation by constructing ambiguity sets built around the empirical measure $\mathbb{P}_n.$ 
The DRO formulation for inverse covariance estimation becomes
\begin{equation}
\label{newDRO}
\min_{K\in\Sdpp}\;\sup_{P:\ \mathcal{D}_{c}(P,\mathbb{P}_{n})\leq \delta}E_{P}[l(W;K)].
\end{equation}
The ambiguity set in this formulation is specified by the collection of measures $\setdef{P}{\mathcal{D}_{c}(P,\mathbb{P}_n)\leq \delta}$, which we now describe.
%
Given two probability distributions $P_1$ and $P_2$ on $\Sd$ and some transportation cost function $c:\Sd\times\Sd\to[0,\infty)$ (which we will specify below), we define the \textit{optimal transport cost} between $P_1$ and $P_2$ as 
\begin{align}
\label{Discrepancy_Def}
\begin{split}
\mathcal{D}_{c}(P_1,P_2) =\inf \{E_{\pi }\left[ c\left(U,V\right) \right]|
\pi \in \mathcal{P}\left(\Sd\times\Sd\right) ,\ \pi_{_U}=P_1, \ \pi _{_V}=P_2\}
\end{split}
\end{align}
where $\mathcal{P}\left(\Sd\times\Sd\right) $ is the
set of joint probability distributions $\pi$ of $(U,V)$ supported on $\Sd\times\Sd$, and $\pi_{_U}$ and $\pi_{_V}$ denote
the marginals of $U$ and $V$ under $\pi $, respectively. In this paper, we are interested in cost functions
\begin{equation}
\label{c_used}
c(U,V)=\norm{\vecc(U)-\vecc(V)}_{q}^\rho,
\end{equation}
with $U,V \in \Sd$, $\rho \geq 1$, $q \in [1,\infty]$. As pointed out by \cite{JB-YK-KM:16}, the resulting optimal transport cost $\mathcal{D}_c^{1/\rho}$ is the Wasserstein distance of order $\rho.$  Our first theoretical result demonstrates that the optimization in \eqref{newDRO} corresponds to a class of regularized estimators under the graphical loss function \eqref{onee11}.

\begin{theorem}[DRO formulation of regularized inverse covariance estimation]
\label{l1}
Consider the cost function  in~\eqref{c_used} for a fixed $\rho\geq 1$. Then,
\begin{align}
\label{eq1}
\begin{split}
\min_{K\in\Sdpp}\;\sup_{P:\ \mathcal{D}_{c}(P,\mathbb{P}_{n})\leq \delta}E_{P} %
\big[l(W;K)\big]=\min_{K\in\Sdpp}\left\{\trace(KA_n)-\log|K|+\delta^{1/\rho}\norm{\vecc(K)}_p\right\},
\end{split}
\end{align}
where  
$\frac{1}{p}+\frac{1}{q}=1$. 
\end{theorem}
%

Theorem~\ref{l1} is a remarkable theoretical result that provides a mapping between the regularization parameter and the uncertainty size $\delta$ of the ambiguity set in the DRO formulation. Then, the regularization problem reduces to determining a good criterion for choosing $\delta$, which we explore in Section~\ref{section3}. Moreover, we obtain the \emph{graphical lasso} formulation~\eqref{zero2} by setting $q=\infty$ in~\eqref{c_used}. 
From~\eqref{eq1}, a smaller ambiguity set implies less robustness being introduced in the estimation problem by reducing the importance of the regularization term. Conversely, a larger regularization term increases the number of nuisance distributions inside the ambiguity set, and thus the robustness.

\begin{remark}
The ambiguity set used in~\eqref{eq1} makes no assumptions on the normality of the distribution of the samples $\{X_1,\dots,X_n\}$. 
Then,~\eqref{eq1} tells us that adding a penalization to the precision matrix gives a robustness in terms of the distributions that the samples may have, which do not necessarily have to be Gaussian for this formulation to hold.  Furthermore, it holds independent of the relationship between $n$ and $d.$
\end{remark}

\section{Selection of the regularization parameter}
\label{section3}

This section follows closely the line of thought recently introduced by~\cite{JB-YK-KM:16} in the analysis of regularized estimators under the DRO formulation. Specifically, we will demonstrate that the ambiguity set $\{P:\ \mathcal{D}_{c}(P,\mathbb{P}_{n})\leq \delta\}$ represents a confidence region for $\Omega = \Sigma\inv,$ and use the techniques of~\cite{JB-YK-KM:16} to explicitly connect the amibiguity size $\delta$ with a confidence level.  As previously stated, $l(W;K)$ is a differentiable 
function on $K\in\Sdpp$ with $\frac{\partial}{\partial K}l(W;K)= W-K^{-1}$, so that 
\begin{equation}
\label{eqDl_zero}
E_{\mathbb{P}_0}\left[\frac{\partial}{\partial K}l(W;K)\Bigr|_{K=\Omega}\right]=\vect{0}_{d\times{d}}.
\end{equation}
Hence, even though the loss function $l(W;K)$ has been inspired from the log-likehood estimation of the covariance matrix $\Sigma$ for samples of Gaussian random vectors, equation~\eqref{eqDl_zero} is transparent to any underlying distribution of the data. For any $K \in \Sdpp,$ define the set
%
\begin{align}
\label{opt-set}
\begin{split}
\mathcal{O}(K):=\Big\{P\in \mathcal{P}(\Sd)\Big|
E_{P}\left[ \frac{\partial}{\partial K'}l(W;K')\Bigr|_{K'=K}\right] = \vect{0}_{d\times{d}}\Big\},
%
\end{split}
\end{align}
corresponding to all probability measures with covariance $K\inv$, i.e.\ for which $K$ is an optimal loss minimization parameter; here, $\mathcal{P}(\Sd)$ denotes the set of all probability distributions supported on $\Sd$. Thus, $\mathcal{O}(\Omega)$ contains all probability measures with covariance matrix agreeing with that of $X.$
%
%

Implicitly, the Wasserstein ambiguity set $\{P:\ \mathcal{D}_{c}(P,\mathbb{P}_{n})\leq \delta\}$ is linked to the collection of covariance matrices
\begin{align}
\label{conf_reg}
\begin{split}
\mathcal{C}_n(\delta):=&\{K\in\Sdpp|\text{ there exists }P\in \mathcal{O}(K) \cap\setdef{P}{\mathcal{D}_{c}(P,\mathbb{P}_{n})\leq \delta}\}\\
=&\bigcup\limits_{P:\ \mathcal{D}_{c}(P,\mathbb{P}_{n})\leq \delta}\arg\min_{K\in\Sdpp}E_P[l(W;K)].
\end{split}
\end{align}
We refer to 
$\mathcal{C}_n(\delta)$ 
as the set of \emph{plausible} selections for $\Omega$.
\begin{lemma}[Interchangeability in the DRO formulation]
\label{l111}
Consider the setting of Theorem~\ref{l1}. 
Then, for $n > d,$ the following holds with probability one:
\begin{align}
\begin{split}
\inf_{K \in \Sdpp}\;\sup_{P:\ \mathcal{D}_{c}(P,\mathbb{P}_{n})\leq \delta}E_{P}%
\left[ l\big( W;K\big)\right]
=\sup_{P:\ \mathcal{D}_{c}(P,\mathbb{P}_{n})\leq \delta}\;\inf_{K\in\Sdpp}E_{P}\left[ l\big(W;K\big)\right]. 
\end{split}
\label{M_M}
\end{align}
\end{lemma}
%
Lemma~\ref{l111} states that any estimator obtained by minimizing the left-hand side of \eqref{M_M} must be in $\mathcal{C}_n(\delta)$, otherwise the right-hand side of \eqref{M_M} would be strictly greater than the left. Thus, in line with the goal of providing a robust estimator, the idea is to choose $\delta$ so that $\mathcal{C}_n(\delta)$ also contains the true inverse covariance matrix $\Omega$ with high confidence.

As $\mathbb{P}_0$ is the weak limit of $\mathbb{P}_n$, we will eventually have that $\Omega \in \mathcal{C}_n(\delta)$ with high probability for any $\delta,$ so that $\mathcal{C}_n(\delta)$ is a confidence region for $\Omega.$ 
From this observation, we can choose the uncertainty size $\delta$ optimally by the criterion
\begin{equation}
\delta=\inf\setdef{\delta>0}{P(\Omega\in\mathcal{C}_n(\delta))\geq 1-\alpha},  \label{OPT_delta1} 
\end{equation}%
i.e., for a specified confidence level $1-\alpha$, we choose $\delta$ so that $\mathcal{C}_n(\delta)$ is a $(1-\alpha)$-confidence region for $\Omega$.
%
%
%

To continue our anlaysis, we make use of the so-called \emph{Robust Wasserstein Profile (RWP) function} $R_n$ introduced by~\cite{JB-YK-KM:16}, 
\begin{equation}
\label{rwp-f}
\begin{split}
R_{n}(K)&=\inf \setdef{\mathcal{D}_{c}(P,\mathbb{P}_n)}{P\in
\mathcal{O}(K)}\\
&=\inf \Big\{\mathcal{D}_{c}(P,\mathbb{P}_{n})\Big|E_{P}\left[\frac{\partial}{\partial K'}l(W;K')\Bigr|_{K'=K}\right]=\vect{0}_{d\times{d}}\Big\},
\end{split}
\end{equation}
for $K \in \Sdpp,$ which has the geometric interpretation of being the minimum distance between the empirical distribution and any distribution that satisfies the optimality condition for the precision matrix $K$. Then, using the equivalence of events $\{\Omega\in\mathcal{C}_n(\delta)\}=\{\mathcal{O}(\Omega)\cap\setdef{P}{\mathcal{D}_{c}(P,\mathbb{P}_n)\leq \delta}\neq\varnothing\}=\{R_{n}(\Omega)\leq\delta\}$,~\eqref{OPT_delta1} becomes equivalent to 
\begin{align}
\label{Opt_Delta_Equiv}
\delta = \arg \inf \setdef{\delta>0}{P(R_{n}(\Omega)\leq\delta)\geq 1-\alpha},
\end{align}
i.e., the optimal selection of $\delta$ is the $1-\alpha$ quantile of $R_{n}(\Omega)$. Indeed, the set $\setdef{P}{\mathcal{D}_{c}(P,\mathbb{P}_n)\leq R_n(\Omega)}$ is the smallest ambiguity set around the empirical measure $\mathbb{P}_n$ such that there exists a distribution for which $\Omega$ is an optimal loss minimization parameter. In contrast to previously reported applications of the RWP function on linear regression and logistic regression~\citep{JB-YK-KM:16}, our problem allows for a (finite sample) closed form expression of this function.  This is due to the fact that we have recast the covariance $\Sigma$ as the mean of the random matrix $XX^\top,$ so that the following result gives a nontrivial generalization of~\cite[Example 3]{JB-YK-KM:16}. 

\begin{theorem}[RWP function]
\label{thm-RWP1} 
Consider the cost function  in~\eqref{c_used} for a fixed $\rho\geq 1$. For $K\in \Sdpp,$ consider $R_n(K)$ as in~\eqref{rwp-f}. Then,
\begin{equation}
\label{RWPI-1}
R_{n}(K) = \norm{\vecc(A_n-K^{-1})}_q^{\rho}.
\end{equation}
\end{theorem}
%
%
%
%
%

We now establish important convergence guarantees on the RWP function in the following corollary.

\begin{corollary}[Asymptotic behavior of the RWP function]
\label{asympt-RWP}
Suppose that the conditions of Theorem~\ref{thm-RWP1} hold, and that $E_{\mathbb{Q}_0}(\norm{X}_2^4) < \infty$. Let $H\in\Sd$ be a matrix of jointly Gaussian random variables with zero mean and such that 
$\mathrm{Cov}(h_{ij},h_{k\ell})
=E[w_{ij}w_{k\ell}]-\Sigma_{ij}\Sigma_{k\ell}
=E[x_{i}x_{j}x_{k}x_{\ell}]-\Sigma_{ij}\Sigma_{k\ell}.$
Then,
\begin{equation}
\label{eq-convRWPI}
n^{\rho/2} R_{n}(\Omega) \Rightarrow \norm{\vecc(H)}_q^\rho.
\end{equation}
\end{corollary}
\begin{proof}
By the central limit theorem, we observe that $\sqrt{n}(A_n-\Sigma)\Rightarrow H$, 
and by the continuous mapping theorem, we get that
\begin{equation*}
n^{\rho/2}R_n(\Omega) = \norm{\sqrt{n}\vecc(A_n-\Sigma)}_q^{\rho}\Rightarrow \norm{\vecc(H)}_q^\rho.
\end{equation*}
\end{proof}

\begin{remark}
\label{remarking}
Turning our attention back to Theorem~\ref{l1}, a robust selection for the ambiguity size or regularization parameter $\lambda=\delta^{1/\rho}$, as obtained from Theorem~\ref{thm-RWP1}, is 
%
\begin{equation}
\label{oii2}
\delta^{1/\rho} = \inf\setdef{\delta>0}{P(\norm{\vecc(A_n-\Sigma)}_q\leq \delta) \geq
1-\alpha}
\end{equation}
%
As a result, this robust selection for $\lambda$ results in a class of estimators, given by minimizers of the right-hand side of \eqref{eq1}, that are invariant to the choice of $\rho$ in \eqref{c_used}.  Thus, for simplicity, we will set $\rho=1$ in the remainder of the paper. 
%
%
%
\end{remark}

\begin{remark}
\label{remarkin}
Let $r_{1-\alpha}$ be the $(1-\alpha)$ quantile from the distribution of the right-hand side of \eqref{eq-convRWPI}.  Then, for any fixed $\alpha,$ the robust selection $\delta$ in \eqref{oii2} satisfies $n^{1/2}\delta  \rightarrow r_{1-\alpha},$ so that the optimal decay rate of $n^{-1/2}$ for $\lambda$ is automatically chosen by the RWP function.
\end{remark}


As solving \eqref{oii2} requires knowledge of $\Sigma$, we now outline the robust selection (RobSel) algorithm for data-adaptive choice of the regularization parameter $\delta$ for our inverse covariance estimation with an $\ell_p$ penalization parameter.  The special case $p=1$ corresponds to the graphical lasso in \eqref{zero2}, in which case we will also use the notation $\delta = \lambda$.  
The asymptotic result in Corollary~\ref{asympt-RWP} invokes a central limit theorem, and thus motivates the approximation of the RWP function through bootstrapping, which we further explain and evaluate its numerical performance in the next section.  Let $\alpha \in (0,1)$ be a prespecified confidence level and $B$ a large integer such that $(B+1)(1 - \alpha)$ is also an integer.
%
\begin{algorithm}
\renewcommand{\thealgorithm}{}
\caption{RobSel algorithm for estimation of the regularization parameter $\lambda$}\label{alg:euclid}
\begin{algorithmic}[1]
\State For $b = 1,\ldots, B$, obtain a bootstrap sample $X_{1b}^*,\ldots, X_{nb}^*$ by sampling uniformly and with replacement from the data, and compute the bootstrap RWP function $R^*_{n,b}=\norm{A^*_{n,b}-A_n}_q$, 
with the empirical covariance $A^*_{n,b}$ computed from the bootstrap sample.
\State Set $\lambda$ to be the bootstrap order statistic $R^*_{n, ((B+1)(1-\alpha))}$.
\end{algorithmic}
\end{algorithm}

RobSel can potentially provide considerable computational savings over cross-validation in practice. Computing sample covariance matrices for each of $B$ bootstrap samples has cost $O(Bnd^2)$. On the other hand, it is known that each iteration of graphical lasso can cost $O(d^3)$ in the worst case~\citep{RM-TH:12}; therefore, performing an $F$-fold cross-validation to search over $L$-grid of regularization parameters, each taking $T$-iterations of graphical lasso, would cost $O(FLTd^3)$.

\section{Numerical results and analysis}
\label{section_num}

The true precision matrix $\Omega\in\Sdpp$ used to generate simulated data has been constructed as follows. First, generate an adjacency matrix of an undirected {Erd\H{o}s-Renyi} graph with equal edge probability and without self-loops. Then, the weight of each edge is sampled uniformly between $[0.5, 1]$, and the sign of each non-zero weight is positive or negative with equal probability of 0.5. Finally, the diagonal entries of this weighted adjacency matrix  are set to $1$ and the matrix is made diagonally dominant by following a procedure described in~\citep{JP-PW-NZ-JZ:09}, which ensures that the resulting matrix $\Omega$ is positive definite. Throughout this numerical study section, a randomly generated sparse matrix $\Omega$ (edge probability $0.1$ and $d=100$) is fixed. 
Using this $\Omega$, a total of $N=200$ datasets (of varying size $n$) were generated as independent observations from a multivariate zero-mean Gaussian distribution, i.e., $\mathcal{N}(\vect{0}_d,\Omega^{-1})$.

Consider the problem of choosing the regularization parameter $\lambda$ (equivalently, the ambiguity size parameter $\delta$) to obtain graphical lasso estimates $\hat{K}_\lambda$ of $\Omega$ using the simulated datasets. An R software package, \texttt{glasso}, from CRAN was used throughout our numerical experiments. Below, we compare two different criteria for choosing  $\lambda$. The first criterion is \emph{Robust Selection (RS)}, 
which follows our proposed RobSel algorithm with $B=200$ sets of bootstrap samples. 
We present here results mainly for $n > d,$ but additional 
results in the high-dimensional regime $n < d$ can be found in the supplementary material.
The second criterion is a \emph{$5$-fold cross-validation (CV)} procedure. The performance on the validation set is the evaluation of the graphical loss function under the empirical measure of the samples on the training set.

%
%
Recall the elements in the confusion matrix to be true positives (TP), true negatives (TN), false positives (FP) and false negatives (FN). 
We compare model selection performance  
of $\lambda$ chosen by the two different approaches: $\lambda_{RS}$ and $\lambda_{CV}$. The following comparison metrics are used:

\begin{itemize}
\item \emph{True positive rate (TPR) and false detection rate (FDR)}: $TPR=\frac{TP}{TP+FN}$ is the proportion of nonzero entries of $\Omega$ that are correctly identified in $\hat K_\lambda$, and $FDR=\frac{FN}{FN+TP}$ is the proportion of zero entries of $\Omega$ that are incorrectly identified as nonzeros in $\hat K_\lambda$.
\item {\emph{Matthew's Correlation Coefficient (MCC)}}: MCC summarizes all counts in the confusion matrix 
in contrast to other measures like TPR and FDR. More details about MCC is given in supplemental subsection~D.1. 
\end{itemize}
%

In the remainder of this section, we compare the model selection performance (FDR, TPR, MCC) from our simulation results. As mentioned in Remark~\ref{remarkin}, supplemental subsection~D.2 shows that $\lambda_{RS}$ decreases as $n$ increases as is also observed to be true with $\lambda_{CV}$. Furthermore, across the tested range of $\alpha$, the regularization $\lambda_{RS}$ are all larger than $\lambda_{CV}$  for any $n$. 
Then, our distributionally robust representation~\eqref{eq1} allows us to 
observe that even for small values of $n$, CV always chooses a $\lambda$ that corresponds to smaller ambiguity sets than RS. 


To assess the accuracy of RobSel in estimating $\lambda = \delta$ for a given $\alpha$, we approximated the right-hand side of \eqref{oii2} using the $N=200$ data sets and the true covariance $\Sigma,$ giving the ``true" value $\lambda_{RWP}$.
Figures in supplemental subsection D.3 show that the performance obtained by $\lambda_{RWP}$ 
is similar to the one obtained by $\lambda_{RS}$ 
for all comparison metrics.  
This finite sample behavior of RS indicates that the RobSel bootstrap algorithm reliably approximates the desired robustness level corresponding to the choice of $\alpha.$ 
These plots also indicate that  
the RS criterion is more conservative than CV in achieving a lower FDR across different sample sizes, due to providing larger values for $\lambda$ (see supplemental subsection D.2). 
More specifically, Fig.~\ref{FDR} shows that RS gives a better performance than CV in terms of FDR even for smaller values of $n$ and this performance improves even more as $n$ increases.

Moreover, the trade-off between the preference for robustness and the preference for a higher density estimation of nonzero entries in the precision matrix can be observed in terms of the Matthews correlation coefficient (MCC), as shown in  Fig.~\ref{Matt_coeff2}. 
Higher values in the curve of MCC implies values of $\lambda$ that describe a better classification of the entries of $\Omega$ as either zero or nonzero 
(see supplemental subsection D.1 for more details). We observe that CV is placed at the left of the optimal value of the MCC and its under-performance is due to the overestimation of nonzero entries in $\Omega$. On the other hand, RS is to the right of the optimal value and its under-performance is due to 
its conservative overestimation of zero entries in $\Omega$ induced by its robust nature. 
Then, it is up to the experimenter to know which method to use depending on whether she desires to control for the overestimation of zero or nonzero entries. Remarkably, for large numbers of $n$, RS seems to be much closer to the optimal performance than CV according to the MCC, and it does this by maintaining a lower FDR than CV while increasing its TPR.

Our results from the MCC analysis and supplemental subsection D.3 also indicate that we should aim for higher values of $\alpha$ if we want a performance closer to CV in terms of TPR when using the RS criterion, with the advantage of still maintaining a better performance than CV in terms of the FDR. In contrast, if we want more conservative results, we should aim for lower values of $\alpha$. This is a good property of RS: it allows the use of a single parameter $\alpha\in(0,1)$ to adjust the importance of the regularization term. 
As mentioned before section 4, a practical importance of RS is that it provides a candidate for $\lambda$ with potentially considerable computational savings over CV. 
%

\section{Conclusion}

We provide a recharacterization of the popular graphical lasso estimator in \eqref{zero2} as the optimal solution to a distributionally robust optimization problem. To the best of our knowledge, this is the first work to make such a connection for sparse inverse covariance estimation. The DRO form of the estimator leads to a reinterpretation of the regularization parameter as the radius of the distributional amibiguity set, which can be chosen based on a desired level of robustness. We propose the RobSel method for the selection of the regularization parameter and compare it to cross-validation for the graphical lasso. In our numerical experiments, RobSel gives a better false detection rate than cross-validation, and, as the sample size increases, other performance metrics like the true positive rate for the two are similar. Moreover, 
%
%
RobSel is a computationally simpler procedure, notably only performing the graphical lasso algorithm once at the final step rather than repeatedly as is necessary for cross-validation. 
Future work includes theoretical justification for robust selection of the regularization parameter for the graphical lasso in the high-dimensional setting.

\begin{figure}[ht]
\setcounter{subfigure}{0}
\centering
\subfloat{\includegraphics[width=0.70\linewidth]{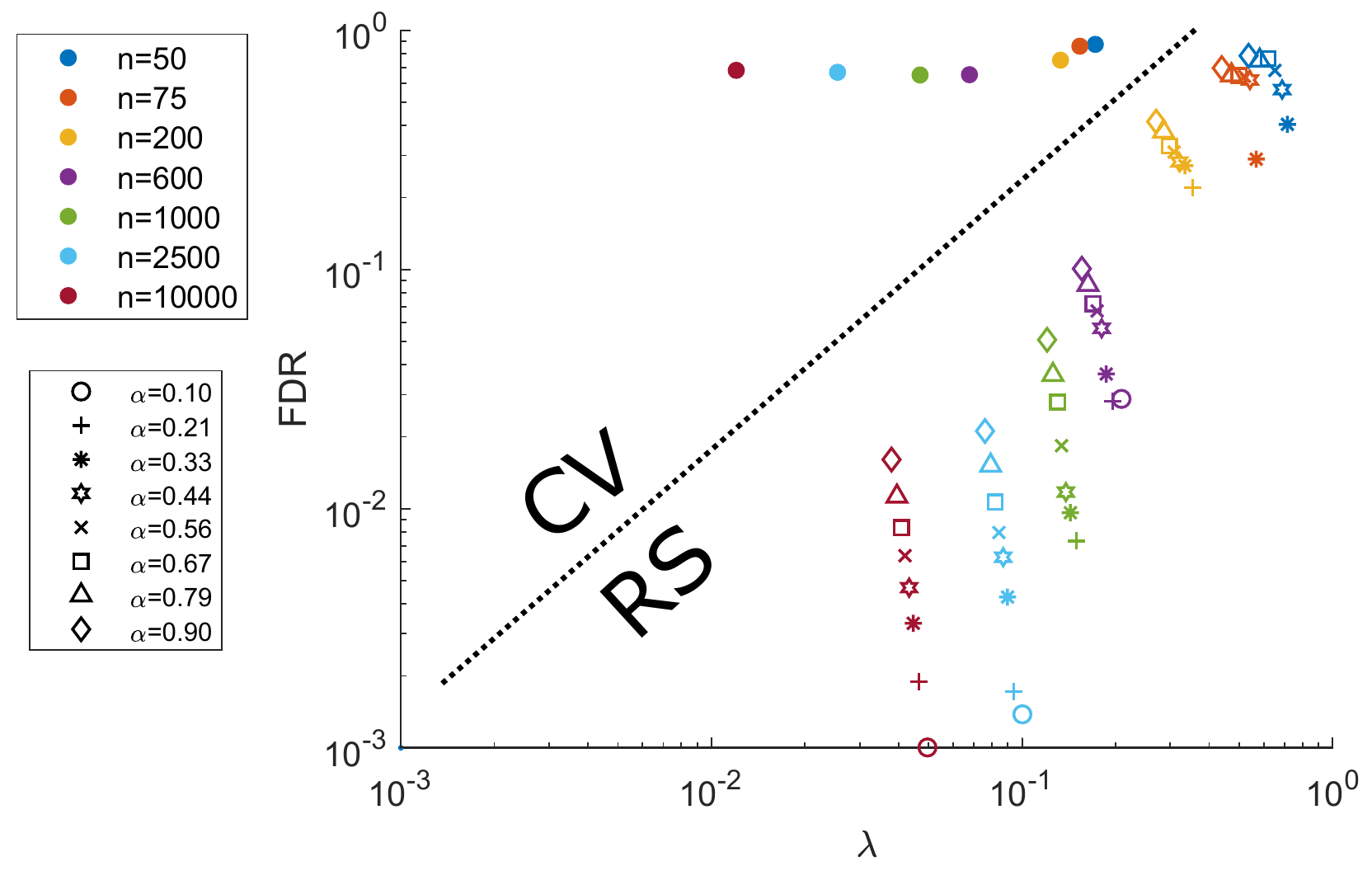}}\\
\caption{
False Detection Rate (FDR) for seven different sample sizes, with both axes in logarithmic scale.  
For each sample size, the average FDR is plotted for both criteria, cross-validation (CV) and RobSel (RS). For RS, a point is plotted with a different symbol for each different value of the parameter $\alpha$ (some points may not be plotted for lower values of $\alpha$, since those values gave no true positive detected, and so FDR was not well-defined).
}
\label{FDR}
\bigbreak
\centering
\subfloat{\includegraphics[width=0.70\linewidth]{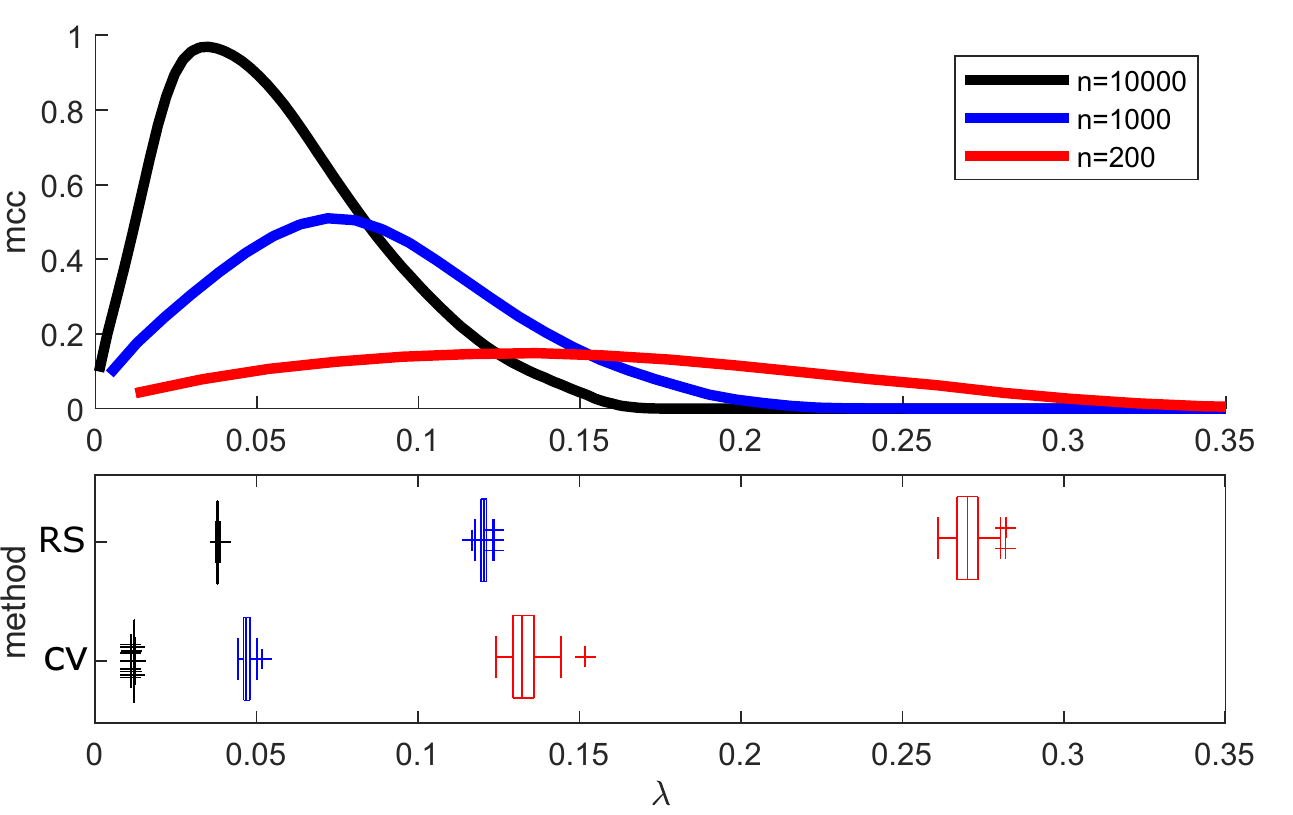}}
\caption{
Matthews 
correlation coefficient (MCC) for three different sample sizes. The curves in the upper plot are the average MCC obtained over $N=200$ datasets.  
The lower plot are boxplots for the CV and RS (with parameter $\alpha=0.9$) methods for the different choices of $n$.
}
\label{Matt_coeff2}
 \end{figure}

\bibliographystyle{abbrvnat}
\bibliography{New}

\clearpage
\appendix

\section{Using a different ambiguity set}
Recall that $X\in\R^{d}$ is a zero-mean random vector with covariance matrix $\Sigma\in\Sdpp$ and measure $\mathbb{Q}_0$, and that we consider an iid random sample $X_1,\cdots,X_n\sim X$, $n > d,$ with empirical measure $\mathbb{Q}_n$. Since we use the graphical loss function as in equation~\eqref{onee1}, we are interested in finding a tractable or closed-form expression for the optimization problem
%
\begin{equation}
\label{oldDRO}
\sup_{Q:\ \mathcal{D}_{c'}(Q,\mathbb{Q}_{n})\leq \delta}E_{Q}[l(X;K)]
\end{equation}
with $K\in\Sdpp$.  Ideally, the solution should be connected to the graphical lasso estimator, since it is one of the most commonly-used sparse inverse covariance estimators in practice.
The ambiguity set in this formulation is specified by the collection of measures $\setdef{Q}{\mathcal{D'}_{c'}(Q,\mathbb{Q}_n)\leq \delta}$, which we now describe. 
Given two probability distributions $Q_1$ and $Q_2$ on $\R^d$ and some transportation cost function $c':\R^d\times\R^d\to[0,\infty)$ (which we will specify below), we define the \textit{optimal transport cost} between $Q_1$ and $Q_2$ as 
\begin{align}
\label{Discrepancy_Def}
\begin{split}
\mathcal{D'}_{c'}(Q_1,Q_2) =\inf \{E_{\pi }&\left[ c'\left(u,v\right) \right]|
\pi \in \mathcal{P}\left(\R^d\times\R^d\right) ,\\
 &\ \pi_{_u}=Q_1, \ \pi _{_v}=Q_2\}
\end{split}
\end{align}
where $\mathcal{P}\left(\R^d\times\R^d\right) $ is the
set of joint probability distributions $\pi$ of $(u,v)$ supported on $\R^d\times\R^d$, and $\pi_{_u}$ and $\pi_{_v}$ denote
the marginals of $u$ and $v$ under $\pi $, respectively. In this paper, we are interested in cost functions
\begin{equation}
\label{c_used2}
c'(u,v)=\norm{u-v}_{q}^\rho,
\end{equation}
with $u,v \in \R^d$, $\rho \geq 1$, $q \in [1,\infty]$.

Now, observe that the function $l(\cdot;K):\R^d\to\R$ is Borel measurable since it is a continuous function. Then, we use the duality result from Proposition 4 of~\cite[version 2]{JB-YK-KM:16} and obtain
\begin{equation}
\label{eq11_old}
\sup_{P:\ \mathcal{D'}_{c'}(Q,\mathbb{Q}_n)\leq \delta}E_{Q} %
\big[l(X;K)\big]=\inf_{\gamma\geq 0}\left\{\gamma\delta+\frac{1}{n}\sum_{i=1}^n\left(\sup_{u\in\R^d}\{l(u;K)-\gamma c'(u,X_i)\}\right)\right\}.
\end{equation}
Let $\Delta:=u-X_i$. Then
\begin{align}
\label{eq12}
\begin{split}
&\sup_{u\in\R^{d}}\{l(u;K)-\gamma c(u,X_i)\} \\
&= \sup_{u\in\R^{d}}\{u^TKu-\log|K|-\gamma\norm{u-X_i}_{q}^\rho\}\\
&=\sup_{ \Delta \in \R^{d}} \{(\Delta + X_i)^TK(\Delta + X_i) - \gamma\norm{\Delta}_q^\rho \} - \log|K|.
%
\end{split}
\end{align}
Replacing this expression back in~\eqref{eq11_old}, it may be difficult, if not impossible, to obtain a closed form optimization problem over $K$. Even if such a simplification is possible, it will not provide the desired connection to the graphical lasso estimator. That is why, in this paper, as outlined in Section~\ref{s_3}, we redefine the ambiguity set to obtain a desired closed form as expressed in Theorem~\ref{l1} in a more transparent way.

\section{Proofs for the paper}

\subsection{Proof of Theorem~\ref{l1}}

\begin{proof}
Consider $K\in\Sdpp$. Observe that the function $l(\cdot;K):\Sd\to\R$ is Borel measurable since it is a continuous function. Then, we use the duality result for the DRO formulation from Proposition~\ref{DRO-dual} from the appendix of this paper and obtain
\begin{equation}
\label{eq11}
\sup_{P:\ \mathcal{D}_{c}(P,\mathbb{P}_n)\leq \delta}E_{P} %
\big[l(W;K)\big]=\inf_{\gamma\geq 0}\left\{\gamma\delta+\frac{1}{n}\sum_{i=1}^n\left(\sup_{W\in\Sd}\{l(W;K)-\gamma c(W,W_i)\}\right)\right\}.
\end{equation}
Let $\Delta:=W-W_i$. Then,
\begin{align}
\label{eq12}
\begin{split}
\sup_{W\in\Sd}&\{l(W;K)-\gamma c(W,W_i)\}\\
&= \sup_{W\in\Sd}\{\trace(KW)-\log|K|-\gamma\norm{\vecc(W)-\vecc(W_i)}_q^\rho\}\\
&= \sup_{\Delta\in\Sd}\{\trace(K(\Delta+W_i))-\gamma\norm{\vecc(\Delta)}_q^\rho\}-\log|K|\\
&= \sup_{\Delta\in\Sd}\{\trace(K\Delta)-\gamma\norm{\vecc(\Delta)}_q^\rho\}+\trace(KW_i)-\log|K|\\
&= \sup_{
\Delta \in \mathcal{M}(K)}\{\norm{\vecc(\Delta)}_q\norm{\vecc(K)}_p-\gamma\norm{\vecc(\Delta)}_q^\rho\}+\trace(KW_i)-\log|K|\\
\end{split}
\end{align}
with $\mathcal{M}(K)=\setdef{\Delta\in\Sd}{\trace(K\Delta) > 0,\, |\Delta_{ij}|^q=\theta |k_{ij}|^p \text{ for some } \theta > 0 }$ 
so that the fourth line follows from selecting a $\Delta\in\Sd$ (since $K\in\Sdpp$) such that Holder's inequality holds tightly (with $\frac{1}{p}+\frac{1}{q}=1$). 
In fact, Holder's inequality holds tightly if and only if 
$\Delta\in\mathcal{M}(K)$~\cite[Chapter 9]{JMS:04}, even for the limiting case $q=\infty$, $p=1$. Observe that 
there exist multiple $\Delta\in\Sd$ that can satisfy Holder's inequality tightly. 
As a consequence, we are still free to choose the magnitude of the $q$-norm of such $\vecc(\Delta)$ (and this is what we will use next).

%
%

Now, the argument inside the supremum in the last line of~\eqref{eq12} is a polynomial function on $\norm{\vecc(\Delta)}_q$. We have to analyze two cases.

\emph{Case 1: $\rho=1$.} In this case we observe that, by setting $\epsilon(\gamma,K)=\sup_{\Delta\in\mathcal{M}(K)}\{\norm{\vecc(\Delta)}_q(\norm{\vecc(K)}_p-\gamma)\}$:
\begin{itemize}
\item if $\gamma\geq\norm{\vecc(K)}_p$, then $\epsilon(\gamma,K)=0$ (in particular, if $\gamma=\norm{\vecc(K)}_p$, the optimizer is $\Delta=\vect{0}_{d\times{d}}$);
\item if $\gamma<\norm{\vecc(K)}_p$, then $\epsilon(\gamma,K)=\infty$;
\end{itemize}
so that, recalling~\eqref{eq11}, due to the outside infimum to be taken over $\gamma\geq 0$; and so we must have that 
\begin{equation*}
\sup_{P:\ \mathcal{D}_{c}(P,\mathbb{P}_n)\leq \delta}E_{P} %
\big[l(W;K)\big]=\inf_{\gamma\geq \norm{\vecc(K)}_p}\left\{\gamma\delta+\frac{1}{n}\sum_{i=1}^n\left(\trace(KW_i)-\log|K|\right)\right\}
\end{equation*}
from which we immediately obtain~\eqref{eq1}.

\emph{Case 2: $\rho>1$.} By differentiation and basic calculus (e.g., using the first and second derivative test) we obtain that the maximizer  $$\Delta^*=\arg\sup_{\Delta\in\mathcal{M}(K)}\{\norm{\vecc(\Delta)}_q\norm{\vecc(K)}_p-\gamma\norm{\vecc(\Delta)}_q^\rho\}$$ is such that $\norm{\vecc(\Delta^*)}_q=\left(\frac{\norm{\vecc(K)}_p}{\gamma\rho}\right)^{\frac{1}{\rho-1}}$. Then, 
\begin{align}
\label{eq13}
\begin{split}
\sup_{W\in\Sd}\{l(W;K)-\gamma c(W,W_i)\} &= \frac{\norm{\vecc(K)}^{\frac{\rho}{\rho-1}}_p}{(\gamma\rho)^{\frac{1}{\rho-1}}}-\gamma\left(\frac{\norm{\vecc(K)}_p}{\gamma\rho}\right)^{\frac{\rho}{\rho-1}}+\trace(KW_i)\\
&\quad-\log|K|\\
&= \norm{\vecc(K)}_p^{\frac{\rho}{\rho-1}}\frac{\rho-1}{\rho^{\frac{\rho}{\rho-1}}\gamma^{\frac{1}{\rho-1}}}+\trace(KW_i)-\log|K|.
\end{split}
\end{align}
Replacing this back in~\eqref{eq11},
\begin{align}
\label{eq14}
\begin{split}
\sup_{P:\ \mathcal{D}_{c}(P,\mathbb{P}_n)\leq \delta}&E_{P} %
\big[l(W;K)\big]\\
&=\inf_{\gamma\geq 0}\left\{\gamma\delta+\frac{1}{n}\sum_{i=1}^n\left(\norm{\vecc(K)}_p^{\frac{\rho}{\rho-1}}\frac{\rho-1}{\rho^{\frac{\rho}{\rho-1}}\gamma^{\frac{1}{\rho-1}}}+\trace(KW_i)-\log|K|\right)\right\}\\
&=\inf_{\gamma\geq 0}\left\{\gamma\delta+\norm{\vecc(K)}_p^{\frac{\rho}{\rho-1}}\frac{\rho-1}{\rho^{\frac{\rho}{\rho-1}}\gamma^{\frac{1}{\rho-1}}}+\trace\left(K\frac{1}{n}\sum_{i=1}^nW_i\right)-\log|K|\right\}\\
&=\inf_{\gamma\geq 0}\left\{\gamma\delta+\norm{\vecc(K)}_p^{\frac{\rho}{\rho-1}}\frac{\rho-1}{\rho^{\frac{\rho}{\rho-1}}\gamma^{\frac{1}{\rho-1}}}\right\}+\trace(KA_n)-\log|K|.
\end{split}
\end{align}
Now, we observe that the argument inside the infimum in the last line of~\eqref{eq14} is a function that grows to infinity when $\gamma\to0$ or $\gamma\to\infty$, so that the minimum is attained for some optimal $\gamma$. By using the first and second derivative tests, we obtain that the minimizer is $\gamma^*=\frac{\norm{\vecc(K)}_p}{\rho\delta^{\frac{\rho-1}{\rho}}}$. Then, replacing this back in~\eqref{eq14} and then this in~\eqref{eq11}, we finally obtain~\eqref{eq1} after some algebraic simplification.
\end{proof}

\subsection{Proof of Lemma~\ref{l111}}
\begin{proof}
Consider $K\in\Sdpp$ and define $g(K)=\sup_{P:\ \mathcal{D}_{c}(P,\mathbb{P}_n)\leq \delta}E_{P}\left[ l\big(W;K\big)\right]$ for a fixed $\delta$. We prove~\eqref{M_M}, by a direct application of Proposition 8 of~\cite[Appendix C]{JB-YK-KM:16}, observing that we satisfy the three conditions for its application:
\begin{enumerate}[label=(\roman*)]
\item $g$ is convex on $\Sdpp$ and finite,\label{one1}
\item there exists $b\in\R$ such that the sublevel set $\kappa_b=\setdef{K}{g(K)\leq b}$ is compact and non-empty, \label{two1}
\item $E_P[l(W;K)]$ is lower semi-continuous and convex as a function of $K$ throughout $\kappa_b$ for any $P\in\setdef{P}{\mathcal{D}_c(P,\mathbb{P}_n)\leq \delta}$. \label{three1}
\end{enumerate}
%
First, we observe that
\begin{align*}
E_{\mathbb{P}_0}[l(W,K)]=E_{\mathbb{P}_0}[\trace(KW)-\log|K|\leq\trace(KE_{\mathbb{P}_0}[W])<\infty,
\end{align*}
since $E_{\mathbb{Q}_0}(\norm{X}^2_2) < \infty$. Then, 
using Theorem~\ref{l1}, the function 
$$
g(K)=\sup_{P:\ \mathcal{D}_{c}(P,\mathbb{P}_n)\leq \delta}E_{P}\left[ l\big(W;K\big)\right]=\trace(KA_n)-\log|K|+\delta^{1/\rho}\norm{\vecc(K)}_p
$$
is finite. Moreover, we also claim it is convex for all $K\in\Sdpp$. This follows from the fact that $\trace(KA_n)-\log|K|$ and $\norm{\vecc(K)}_p$, $p\in[1,\infty]$ are two convex functions on $K\in\Sdpp$, and from the fact that
the  nonnegative weighted sum of two convex functions is another convex function~\citep{SB-LV:04}. This proves~\ref{one1}. 
%

Now, we claim that $g(K)$ is radially unbounded, i.e., $g(K)\to\infty$ as $\norm{\vecc(K)}_p\to\infty$. To see this, recall 
that $\trace(KA_n)-\log|K|$ is a differentiable convex function in $K$ that is minimized whenever $K^{-1}=A_n$, since $A_n$ is invertible for $n>d$ almost surely. Then,
$$
g(K)=\trace(KA_n)-\log|K|+\delta^{1/\rho}\norm{\vecc(K)}_p\geq d-\log|A_n^{-1}|+\delta^{1/\rho}\norm{\vecc(K)}_p,
$$
from which it immediately follows that $g(K)\to\infty$ as $\norm{\vecc(K)}_p\to\infty$.


Now, again, since $g(K)$ is also convex and continuous in $\Sdpp$, we conclude that the level sets $\kappa_b=\setdef{K}{g(K)\leq b}$ are compact and nonempty as long as $b>l(W;K)+\delta^{1/\rho}\norm{\vecc(K)}_p$.
 This proves~\ref{two1}. Moreover, since $l(W;K)$ is convex and continuous on any $K\in\Sdpp$, it follows that $E_P[l(W;K)]$ for any $P\in\setdef{P}{\mathcal{D}_c(P,\mathbb{P}_n)\leq \delta}$ is also continuous and convex on any $K\in\Sdpp$, thus \ref{three1} follows immediately.

%
\end{proof}

\subsection{Proof of Theorem~\ref{thm-RWP1}}

\begin{proof}

Consider $K\in\Sdpp$. Setting $h(U;K)=U-K^{-1}$, it is clear that we satisfy the conditions for applying Proposition~\ref{RWP-duality} in the Appendix, and so we obtain 
\begin{equation}
\label{princip}
R_n(K)=\sup_{\Lambda\in\Sd}\left\{-\frac{1}{n}\sum^n_{i=1}\sup_{U\in\Sd}\{\trace(\Lambda^\top(U-K^{-1}))-\norm{\vecc(U)-\vecc(W_i)}_q^{\rho}\}\right\}
\end{equation}
Now, letting $\Delta:=U-W_i^\top$
\begin{align}
\label{oinki12}
\begin{split}
\sup_{U\in\Sd}\{\trace(\Lambda^\top(U-K^{-1}))-&\norm{\vecc(U)-\vecc(W_i)}_q^{\rho}\}\\
&=
\sup_{\Delta\in\Sd}\{\trace(\Lambda^\top(\Delta+W_i-K^{-1}))-\norm{\vecc(\Delta)}_q^{\rho}\}\\
&=\sup_{\Delta\in\Sd}\{\trace(\Lambda^\top\Delta)-\norm{\vecc(\Delta)}_q^{\rho}\}+\trace(\Lambda^\top(W_i-K^{-1}))\\
&=\sup_{\substack{\Delta\in\mathcal{M}(\Lambda)}}\{\norm{\vecc(\Lambda)}_p\norm{\vecc(\Delta)}_q-\norm{\vecc(\Delta)}_q^{\rho}\}\\
&+\trace(\Lambda^\top(W_i-K^{-1}))
\end{split}
\end{align}
with $\mathcal{M}(\Lambda)$ as in the proof of Theorem~\ref{l1}, 
so that the third line follows from selecting a $\Delta\in\Sd$ such that Holder's inequality holds tightly (with $\frac{1}{p}+\frac{1}{q}=1$), whose existence has been explained in the proof of Theorem~\ref{l1}. Thus, we are still free to choose the magnitude of the $q$-norm of such $\vecc(\Delta)$ (and this is what we will use next).

Now, the argument inside the supremum in the last line of~\eqref{oinki12} is a polynomial function on $\norm{\vecc(\Delta)}_q$. We have to analyze two cases.

\emph{Case 1: $\rho=1$.} In this case we observe that, by setting $\epsilon(\Lambda)=\sup_{\Delta\in\Sd}\{\norm{\vecc(\Delta)}_q(\norm{\vecc(\Lambda)}_p-1)\}$:
\begin{itemize}
\item if $\norm{\vecc(\Lambda)}_p\leq 1$, then $\epsilon(\Lambda)=0$ (in particular, if $\norm{\vecc(\Lambda)}_p<1$, the optimizer is $\Delta=\vect{0}_{d\times{d}}$);
\item if $\norm{\vecc(\Lambda)}_p>1$, then $\epsilon(\Lambda)=\infty$;
\end{itemize}
so that, recalling~\eqref{princip}, we see that if $\norm{\vecc(\Lambda)}_p>1$, then $R_n(K)=-\infty$. Then, we obtain that
\begin{align*}
R_n(K)&=\sup_{\Lambda\in\Sd:\norm{\vecc(\Lambda)}_p\leq 1}\left\{-\frac{1}{n}\sum^n_{i=1}\trace(\Lambda^\top(W_i-K^{-1})\right\}\\
&=\sup_{\Lambda\in\Sd:\norm{\vecc(\Lambda)}_p\leq 1}\left\{-\trace(\Lambda^\top(A_n-K^{-1})\right\}\\
&=\sup_{\Lambda\in\Sd:\norm{\vecc(\Lambda)}_p\leq 1}\left\{\vecc(\Lambda)^\top\vecc(A_n-K^{-1})\right\}\\
&= \norm{\vecc(A_n-K^{-1})}_q
\end{align*}
where the third line results from the fact that $\Lambda$ is a free variable so we can flip its sign, and the last line follows from the the analysis of conjugate norms and the fact that $\Lambda,A_n-K^{-1}\in\Sd$. We thus obtained~\eqref{RWPI-1}.

\emph{Case 2: $\rho>1$.} By differentiation and basic calculus (e.g., using the first and second derivative test) we obtain that the maximizer  $$\Delta^*=\arg\sup_{\Delta\in\Sd}\{\norm{\vecc(\Delta)}_q\norm{\vecc(\Lambda)}_p-\gamma\norm{\vecc(\Delta)}_q^\rho\}$$ is such that $\norm{\vecc(\Delta^*)}_q=\left(\frac{\norm{\vecc(K)}_p}{\rho}\right)^{\frac{1}{\rho-1}}$. Then, replacing this back in~\eqref{princip},
\begin{align*}
\begin{split}
R_n(K)&=\sup_{\Lambda\in\Sd}\left\{-\frac{1}{n}\sum^n_{i=1}\left(\norm{\vecc(\Lambda)}^{\frac{\rho}{\rho-1}}_p\frac{\rho-1}{\rho^{\frac{\rho}{\rho-1}}}+\trace(\Lambda^\top(W_i-K^{-1}))\right)\right\}\\
&=\sup_{\Lambda\in\Sd}\left\{-\norm{\vecc(\Lambda)}^{\frac{\rho}{\rho-1}}_p\frac{\rho-1}{\rho^{\frac{\rho}{\rho-1}}}-\trace(\Lambda^\top(A_n-K^{-1}))\right\}\\
&=\sup_{\Lambda\in\Sd}\left\{\trace(\Lambda^\top(A_n-K^{-1}))-\norm{\vecc(\Lambda)}^{\frac{\rho}{\rho-1}}_p\frac{\rho-1}{\rho^{\frac{\rho}{\rho-1}}}\right\}\\
&=\sup_{\Lambda\in\Sd}\left\{\norm{\vecc(\Lambda)}_p\norm{\vecc(A_n-K^{-1})}_q-\norm{\vecc(\Lambda)}^{\frac{\rho}{\rho-1}}_p\frac{\rho-1}{\rho^{\frac{\rho}{\rho-1}}}\right\}.
\end{split}
\end{align*}
Again, by differentiation and basic calculus, we obtain that the maximizer $\Lambda^*$ is such that $\norm{\vecc(\Lambda^*)}_p=\rho\norm{\vecc(A_n-K^{-1})}_q^{\rho-1}$. Replacing this value back in our previous expression, we get that $R_n(K)=\norm{\vecc(A_n-K^{-1})}_q^{\rho}$, and thus we showed~\eqref{RWPI-1}. 
\end{proof}

\section{Applicability of the dual representations of the RWP function and the DRO formulation}
The dual representations of the RWP function and the DRO formulation for the case in which the space of probability measures is $\mathcal{P}(\R^d\times{\R^d})$ is studied in the paper~\citep{JB-YK-KM:16}. In this paper, we are interested in the case $\mathcal{P}(\Sd\times{\Sd})$. In other words, we consider the samples to be in $\Sd$ instead of $\R^d$. We want to emphasize that the derivations of these dual representations rely on the dual formulation of the so called ``problem of moments" or a specific class of ``Chebyshev-type inequalities" referenced in the work by \cite{KI:62}. The derivation by Isii is actually more general in the sense that is applied to more general probability spaces than the ones used in this paper and in~\citep{JB-YK-KM:16} (in fact, it is stated for general spaces of non-negative measures).

Throughout this section, we consider an integrable function $h:\Sd\times\Sd\to\Sd$, and a lower semi-continuous function $c:\Sd\times{\Sd}\to[0,\infty)$ such that $c(U,U)=0$ for any $U\in\Sd$ and such that the set 
$$
\Omega := \setdef{(U',W')\in\Sd\times\Sd}{c(U',W')<\infty}
$$
is Borel measurable and non-empty. Also consider an iid random sample $W_1,\cdots,W_n\sim W$ with $W$ coming from a distribution on $\mathcal{P}(\Sd)$.  
%
 
Now, let us focus first on the RWP function in the following proposition which parallels~\cite[Proposition 3]{JB-YK-KM:16}.

\begin{proposition}
\label{RWP-duality} 
Consider $K\in\Sdpp$. Let $h(\cdot,K)$ be Borel measurable. Also, suppose that $\vect {0}_{d\times{d}}$ lies in the interior of the convex hull of $\setdef{h(U',C)}{U'\in\Sd}$. Then, 
\begin{align*}
R_n(K) = \sup_{\Lambda\in\Sd}\left\{-\frac{1}{n}\sum^n_{i=1}\sup_{U\in\Sd}\{\trace(\Lambda^\top h(U;K))-c(U,W_i)\}\right\}.
\end{align*}
\end{proposition}
\begin{proof}
Consider the proof of~\cite[Proposition 3]{JB-YK-KM:16}. If we:
\begin{itemize}
    \item set the estimating equation by $E[h(W;K)]=\vect{0}_{d\times d}$,
    \item set 
\begin{equation*}
R_n(K)=\inf\setdef{E_{\pi}[c(U,W)]}{E_{\pi}[h(U;K)]=\vect{0}_{n\times{n}},\pi_W=\mathbb{P}_n,\pi\in\mathcal{P}(\Sd\times{\Sd})},
\end{equation*}
with $\pi_W$ denoting the marginal distribution of $W$,
\item consider the previously defined $\Omega$,
\end{itemize}
then 
%
%
we obtain that, following the rest of this proof (and using~\cite[Theorem 1]{KI:62} with its special case):
\begin{align*}
R_n(K)&=\sup_{\Lambda\in\Sd}\left\{-\frac{1}{n}\sum^n_{i=1}\sup_{U\in\Sd}\{\vecc(\Lambda)^\top\vecc(h(U;K))-c(U,W_i)\}\right\}\\
&=\sup_{\Lambda\in\Sd}\left\{-\frac{1}{n}\sum^n_{i=1}\sup_{U\in\Sd}\{\trace(\Lambda^\top h(U;K))-c(U,W_i)\}\right\},
\end{align*}
thus obtaining the dual representation of the RWP function. 
\end{proof}

The following proposition for the dual representation of the DRO formulation parallels~\cite[Proposition 1]{JB-YK-KM:16}.

\begin{proposition} 
\label{DRO-dual} 
For $\gamma \geq 0$ and loss functions $l(U';K)$ that are upper semi-continuous in $U'\in\Sd$ for each $K\in\Sdpp$, let 
\begin{equation}  \label{Adjust-Infty}
\phi _{\gamma }(W_{i};K)=\sup_{U\in\Sd}\left\{l(U;K)-\gamma c(U,W_{i}))\right\}.
\end{equation}
Then 
\begin{equation*}
\sup_{P:\ \mathcal{D}_{c}(P,\mathbb{P}_n)\leq \delta }E_P\big[l(W;K)\big]
=\min_{\gamma \geq 0}\left\{ \gamma \delta +\frac{1}{n}\sum_{i=1}^{n}\phi
_{\gamma}(W_{i};K)\right\} .
\end{equation*}
\end{proposition}
\begin{proof}
The proof for the dual representation of the DRO for our domain of symmetric matrices is also very similar to the one described in Proposition 4 of~\cite[version 2]{JB-YK-KM:16}, just by following appropriate similar changes as we did for the proof of~\ref{RWP-duality}.
%
%
%
\end{proof}

\section{Figures, tables and additional analysis for the numerical results (Section 4 of the paper)}

\subsection{Matthews correlation coefficient analysis}\label{section:mcc}

Let true positives (TP) be the number of nonzero off-diagonal entries of $\Omega$ that are correctly identified, 
false negatives (FN) be the number of its nonzero off-diagonal entries that are incorrectly identified as zeros, false positives (FP) be the number of its zero off-diagonal entries that are incorrectly identified as nonzeros, and true negatives (TN) be the number of its zero off-diagonal entries that are correctly identified. 
%
Given the estimated precision matrix $\hat{K}$, the Matthews correlation coefficient (MCC)~\citep{P:11} is defined as:
\begin{equation}
MCC = \frac{TP\cdot TN-FP\cdot FN }{\sqrt{(TP+FP)(TP+FN)(TN+FP)(TN+FP)}},    
\end{equation}
and, whenever the denominator is zero, it can be arbitrarily set to one.
It can be shown that $MCC\in[-1,+1]$, and $+1$ is interpreted as a perfect prediction (of both zero and nonzero values), $0$ is interpreted as prediction no better than a random one, and $-1$ is interpreted as indicating  total disagreement between prediction and observation.

MCC has been argued to be one of the most informative coefficients for assessing the performance of binary classification (in this case, classifying if an entry of the precision matrix is zero or nonzero) since it summarizes all information from the TP, TN, FP and FN quantities~\citep{Chicco2017,P:11}, in contrast to other measures like TPR and FPR.

\subsection{Regularization parameters}
\label{supp:regularization-parameters}
All the plots related to the RS criterion for choosing $\lambda$ have in their x-axis the values\\ $\alpha\in\{0.10,0.21,0.33,0.44,0.56,0.67,0.79,0.90\}$. We study the cases for sample sizes $n\in\{75,200,1000\}$.

\begin{figure}[H]
\setcounter{subfigure}{0}
\centering
\subfloat[RS]{\includegraphics[width=0.33\linewidth]{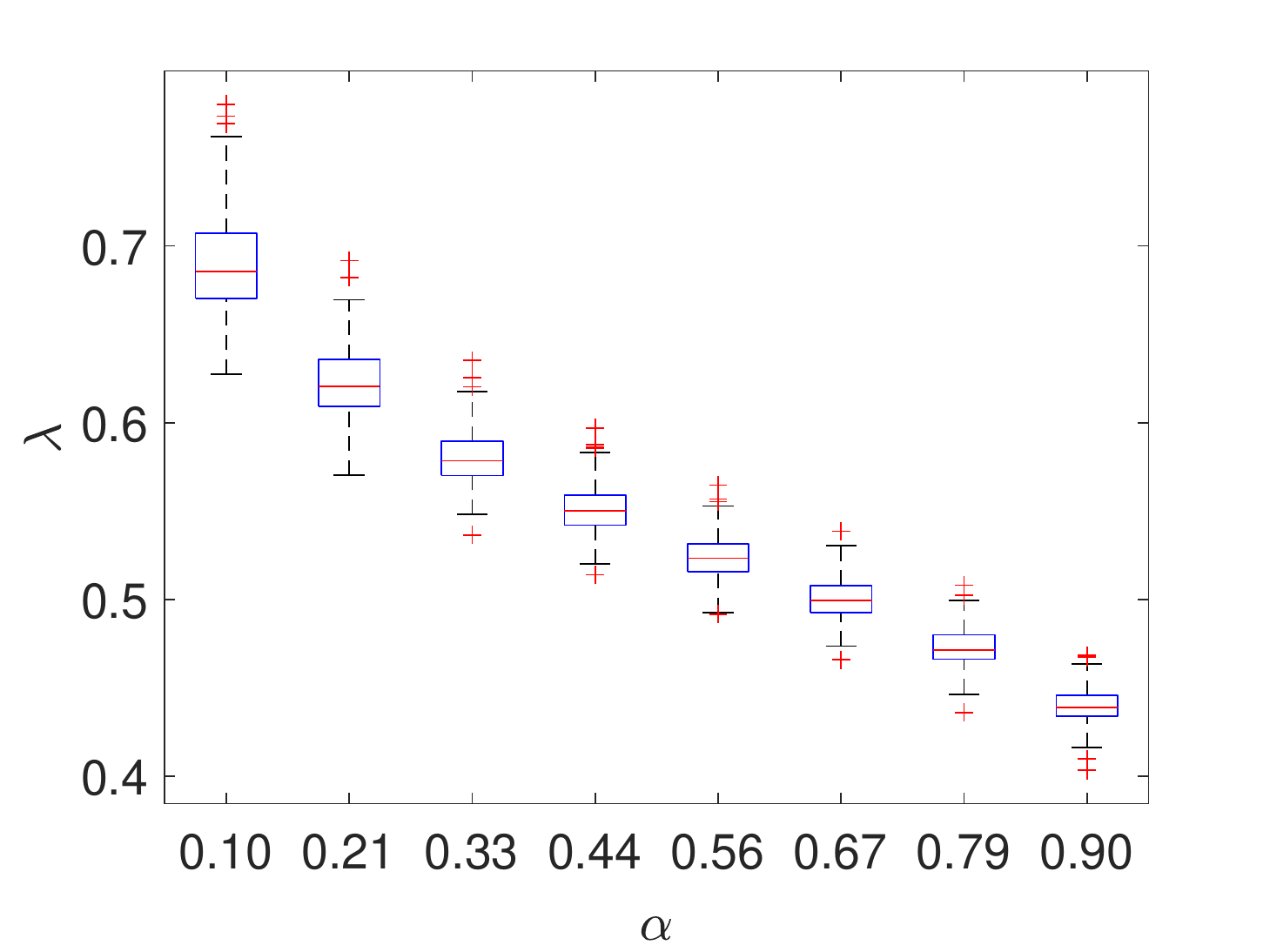}}
\subfloat[CV]{\includegraphics[width=0.33\linewidth]{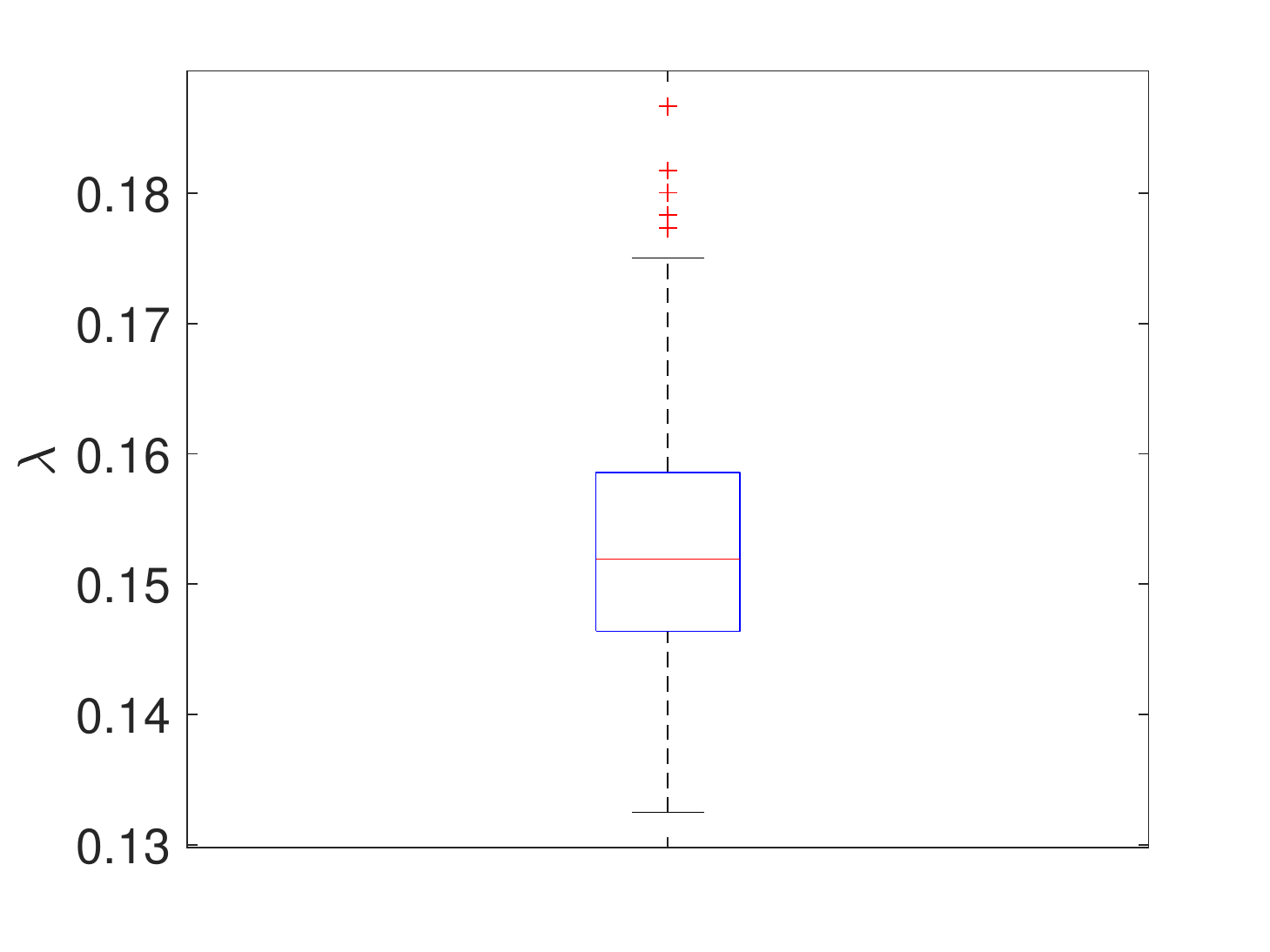}}
\caption{$n=75$}
 \end{figure}
 
 \begin{figure}[H]
\setcounter{subfigure}{0}
\centering
\subfloat[RS]{\includegraphics[width=0.33\linewidth]{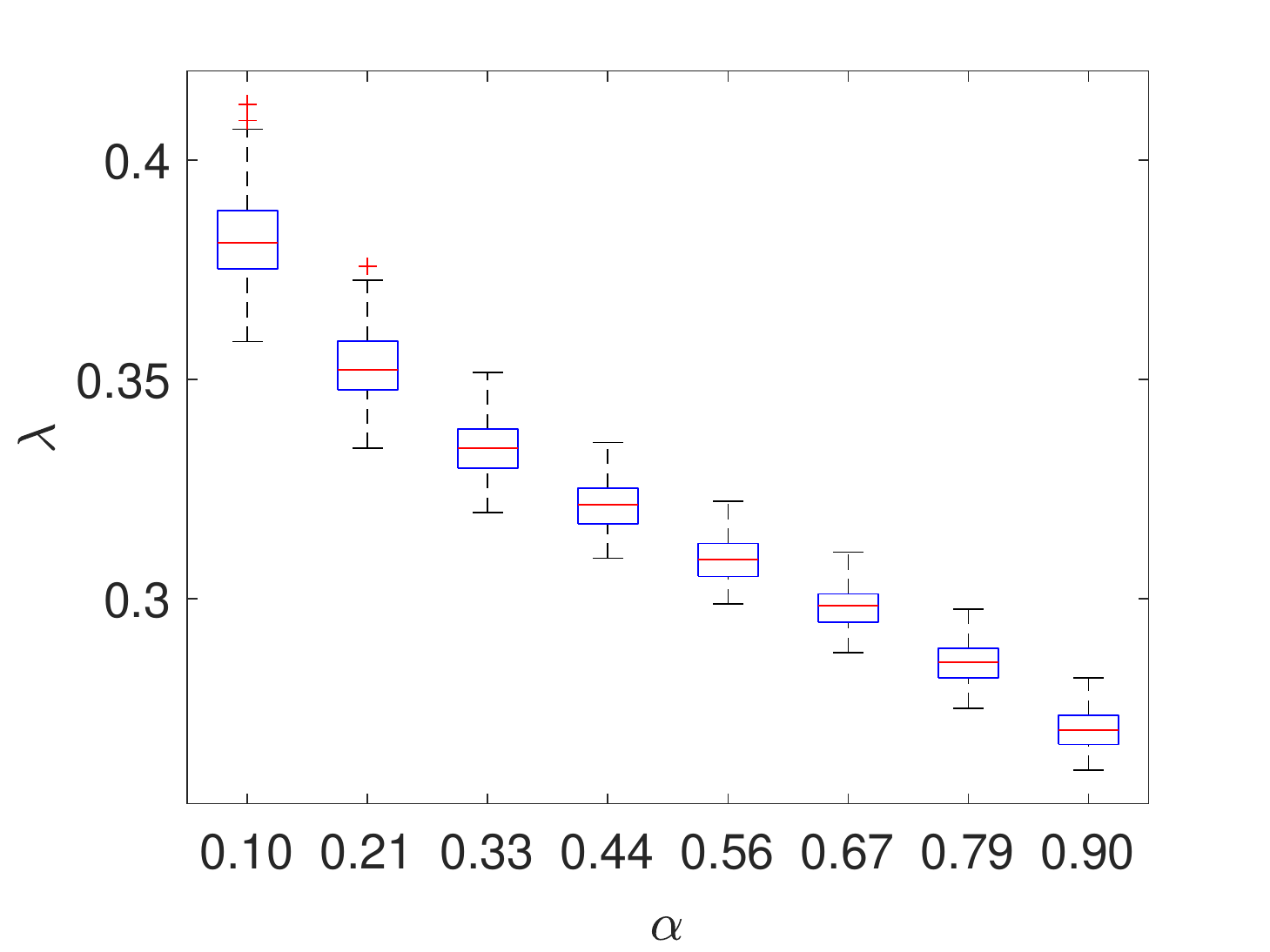}}
\subfloat[CV]{\includegraphics[width=0.33\linewidth]{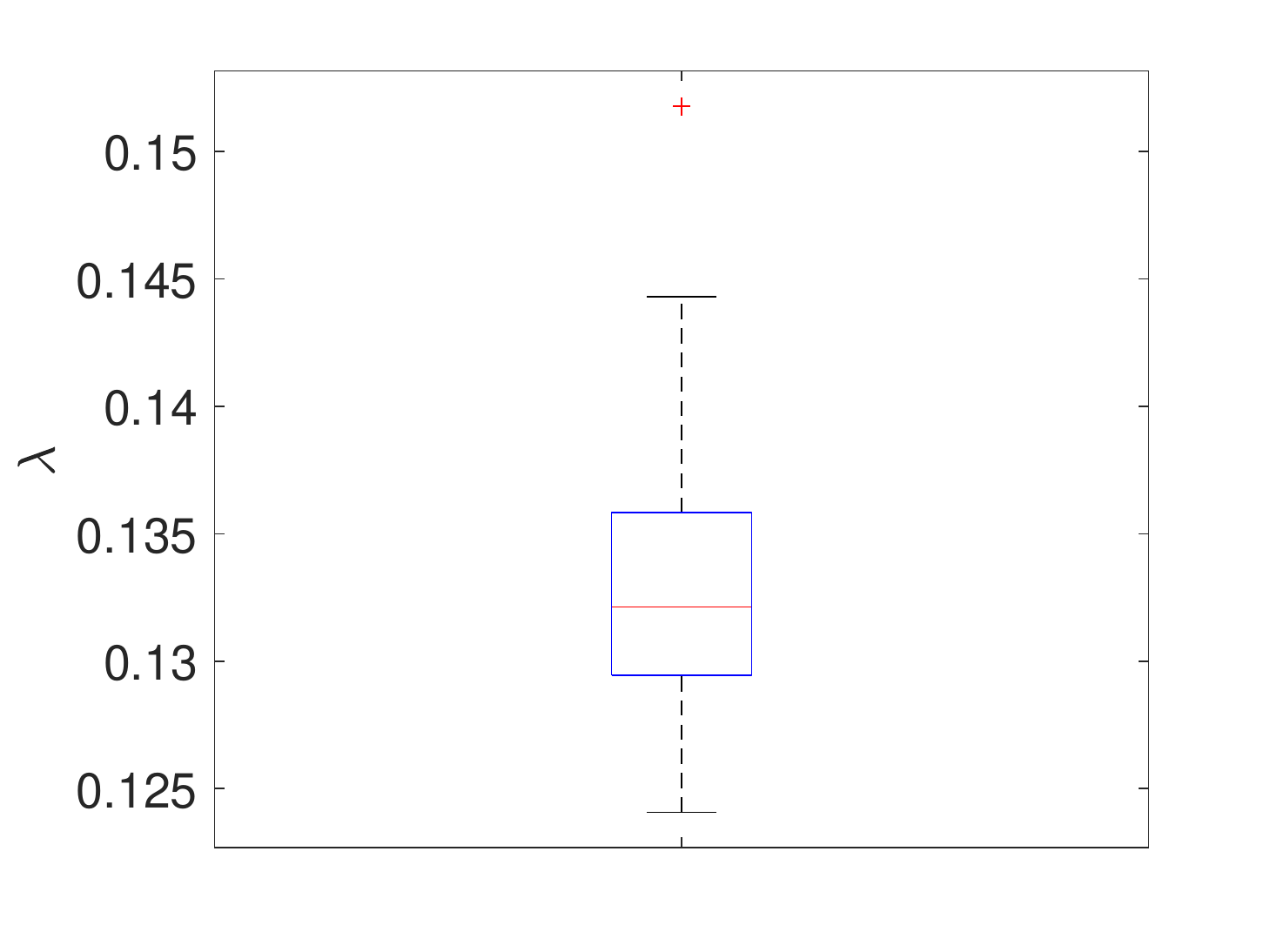}}
\caption{$n=200$}
 \end{figure}
 
  \begin{figure}[H]
\setcounter{subfigure}{0}
\centering
\subfloat[RS]{\includegraphics[width=0.33\linewidth]{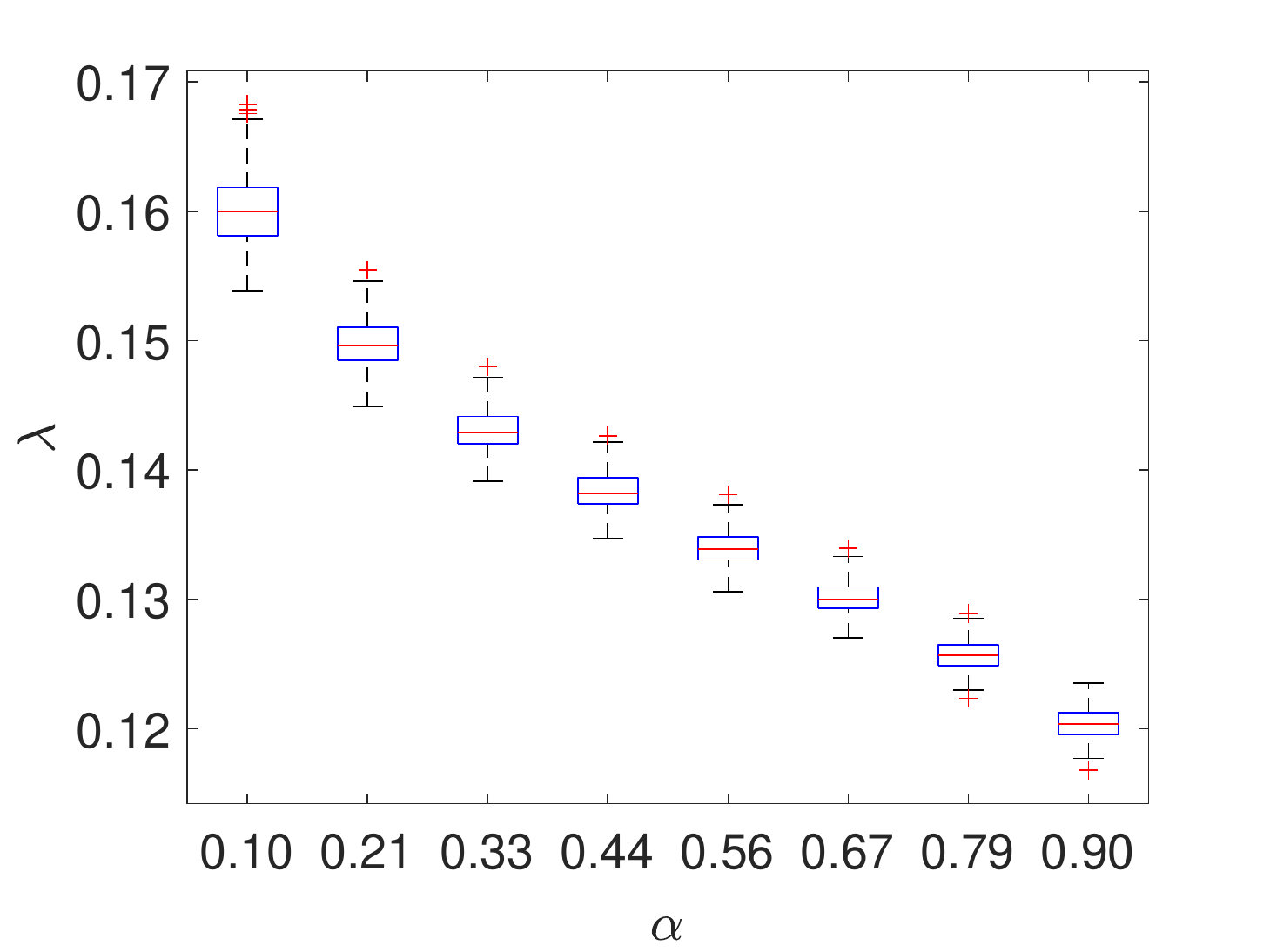}}
\subfloat[CV]{\includegraphics[width=0.33\linewidth]{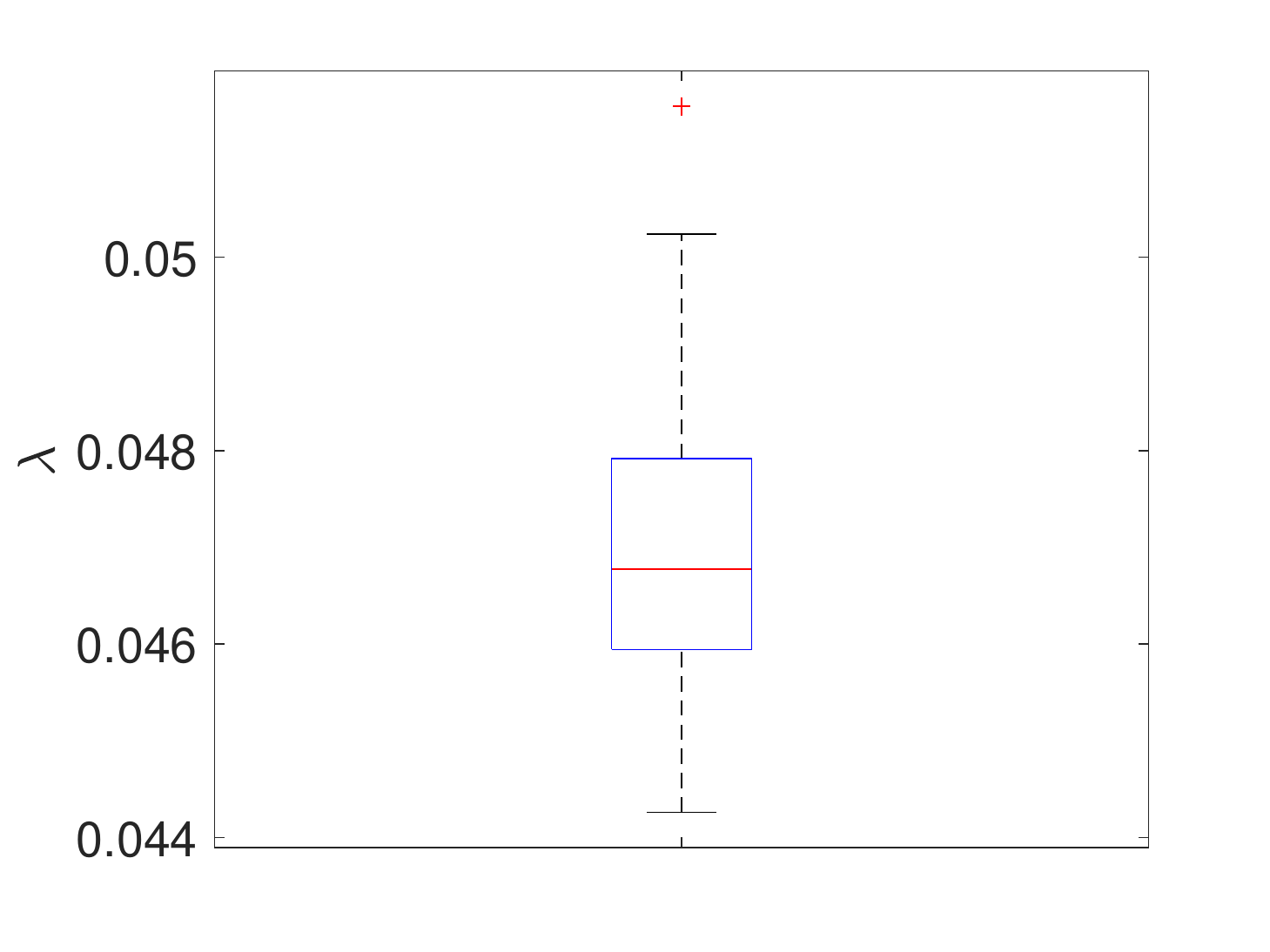}}
\caption{$n=1000$}
 \end{figure}

%

\subsection{Performance figures for different choices of the regularization parameter}
\label{supp:performance-metrics}

All the plots related to the RS and RWP criteria for choosing $\lambda$ have in their x-axis the values \\$\alpha\in\{0.10,0.21,0.33,0.44,0.56,0.67,0.79,0.90\}$. We study the cases for sample sizes $n\in\{75,200,1000\}$.

\subsubsection{$n=75$}

\begin{figure}[H]
\setcounter{subfigure}{0}
\centering
\subfloat[RS]{\includegraphics[width=0.27\linewidth]{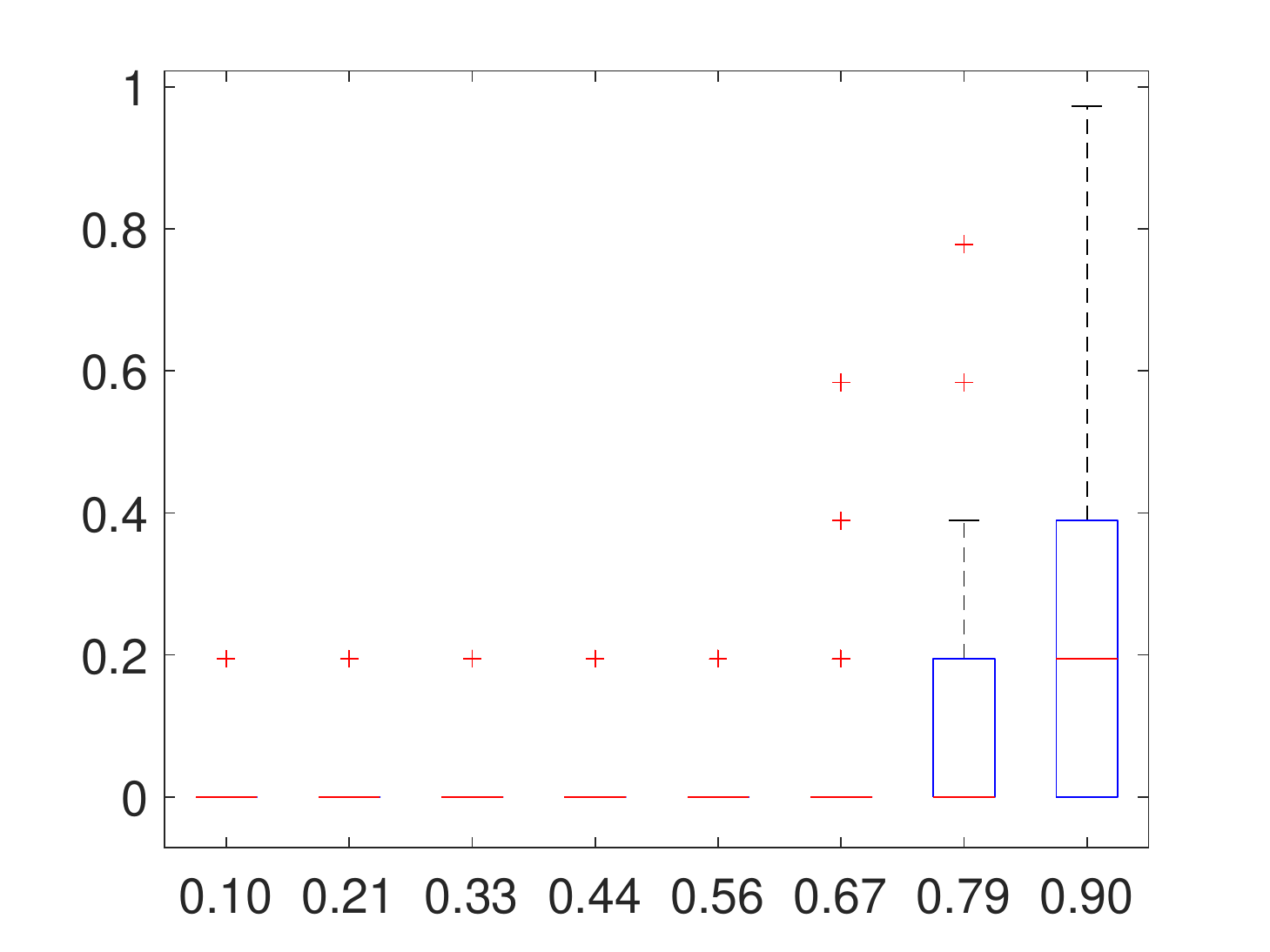}}
\subfloat[CV]{\includegraphics[width=0.27\linewidth]{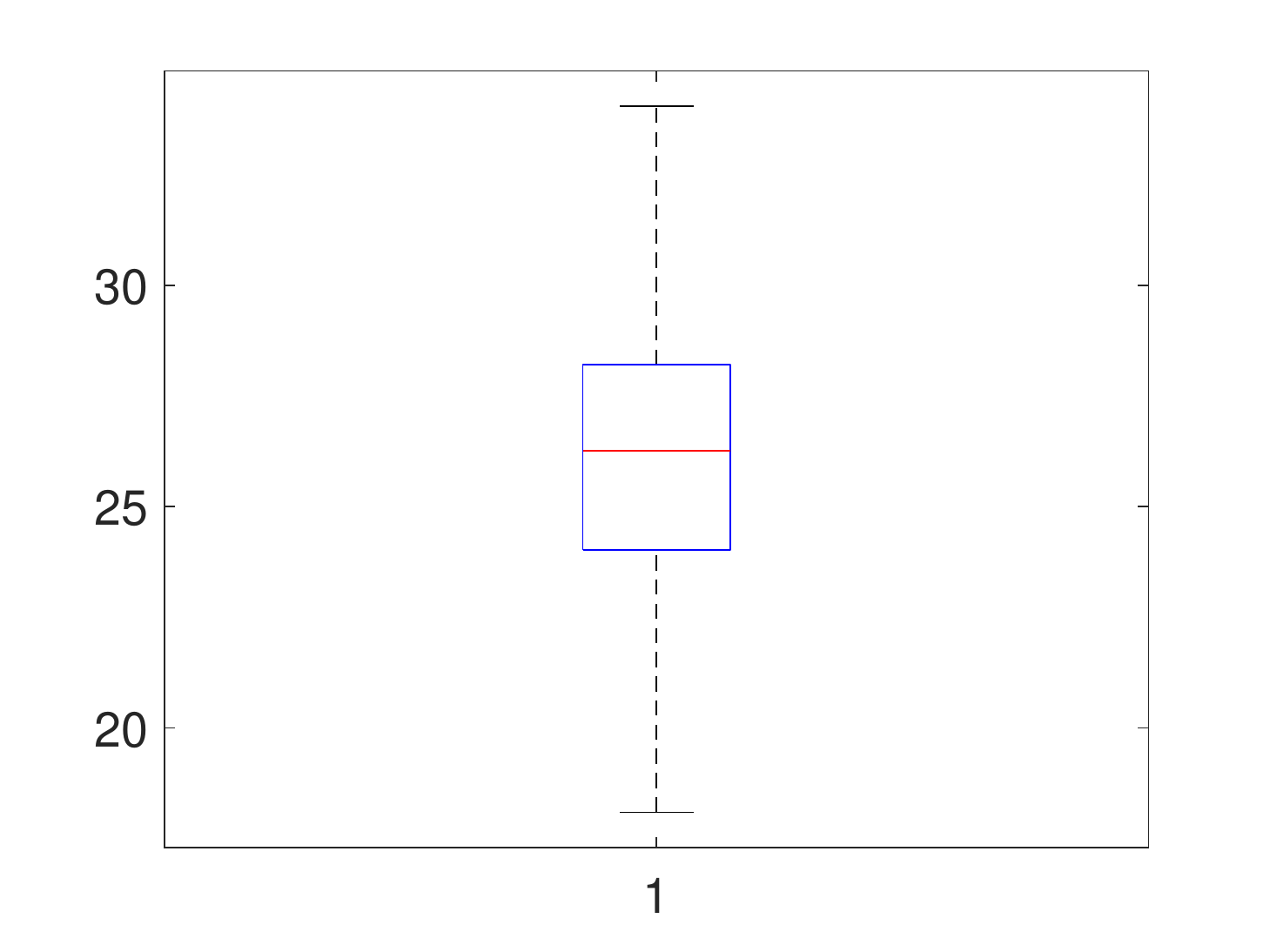}}
\subfloat[RWP]{\includegraphics[width=0.27\linewidth]{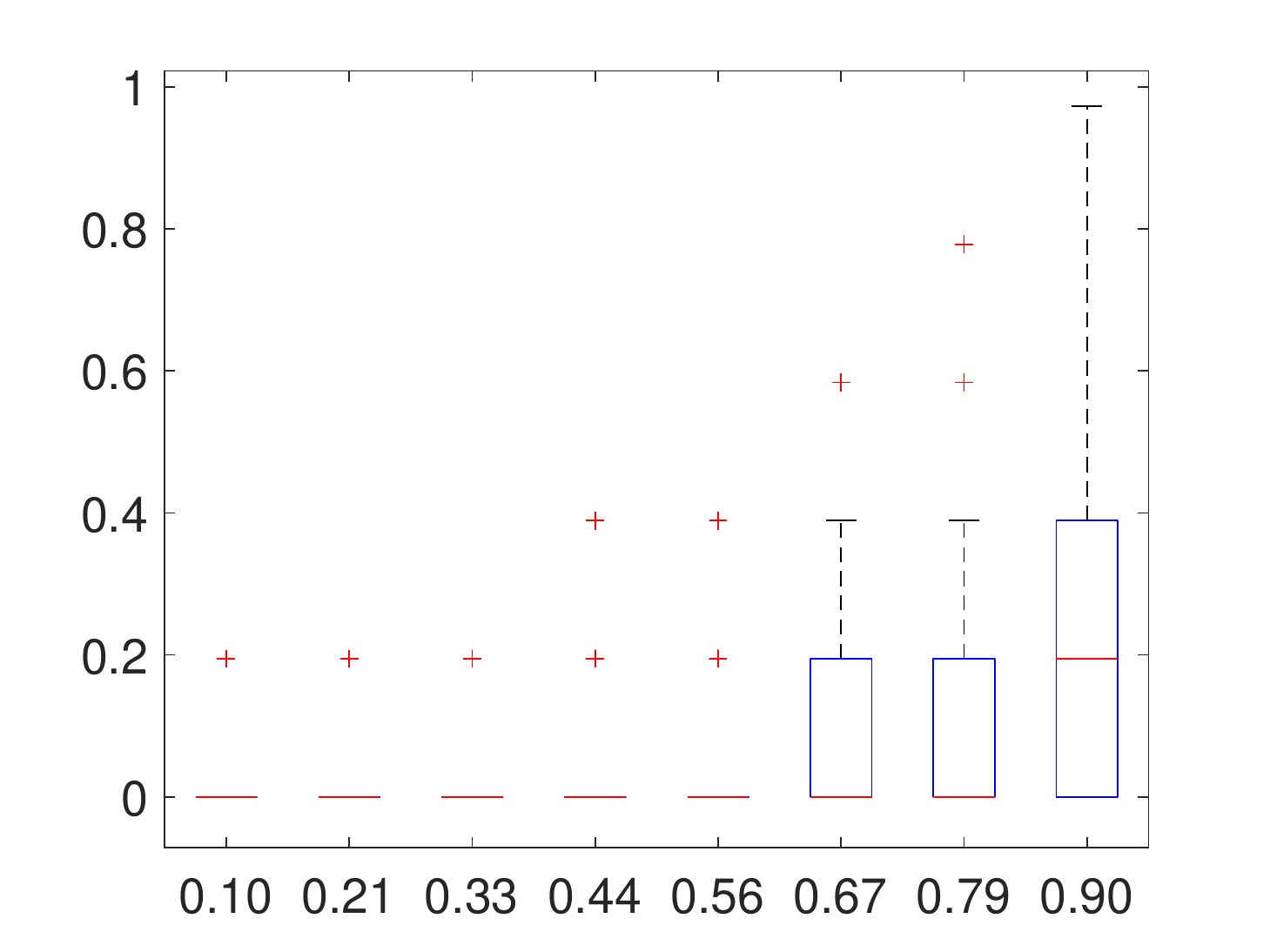}}
\caption{True positive rate ($\%$)}
 \end{figure}

%
\begin{figure}[H]
\setcounter{subfigure}{0}
\centering
\subfloat[RS]{\includegraphics[width=0.27\linewidth]{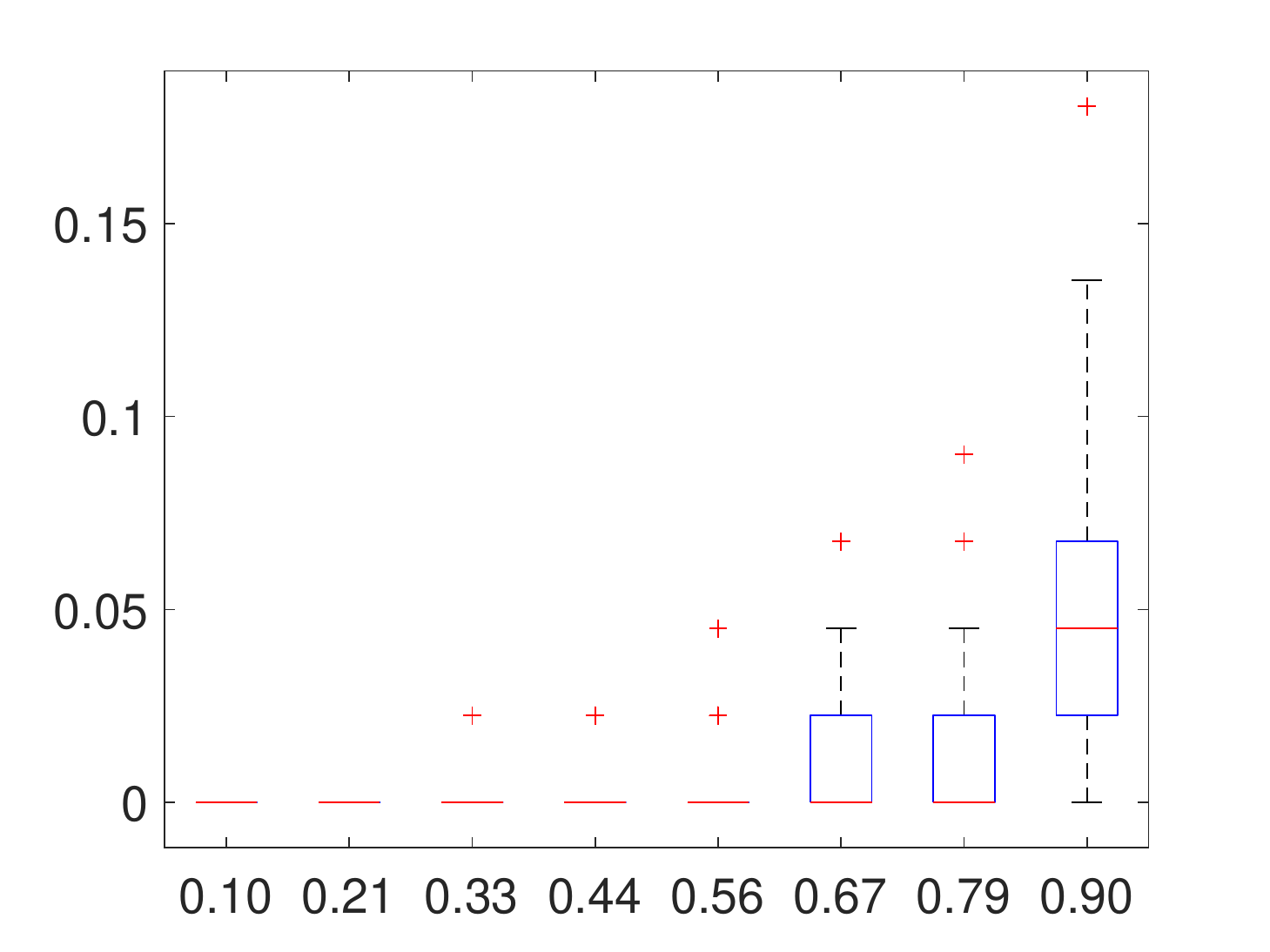}}
\subfloat[CV]{\includegraphics[width=0.27\linewidth]{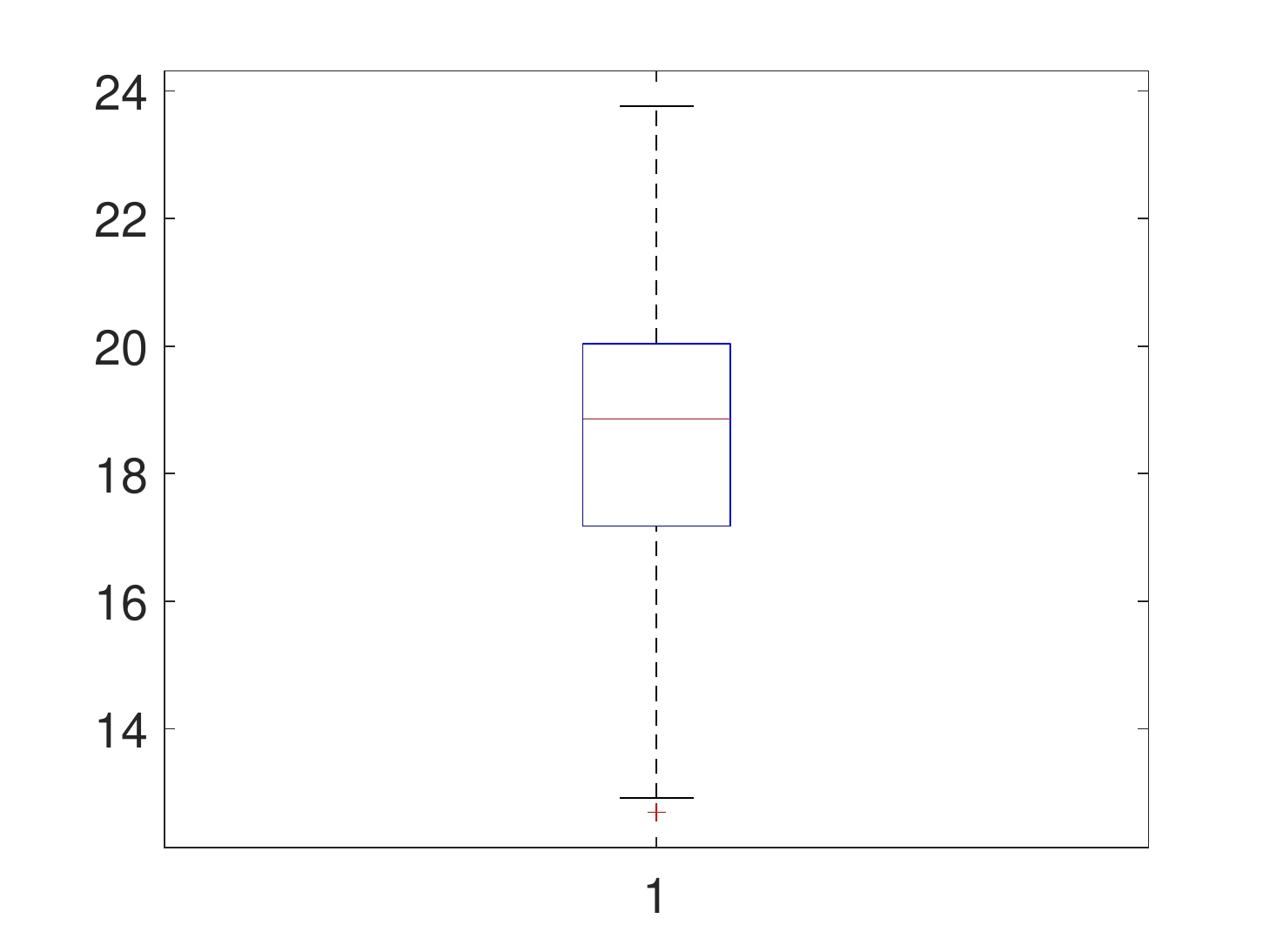}}
\subfloat[RWP]{\includegraphics[width=0.27\linewidth]{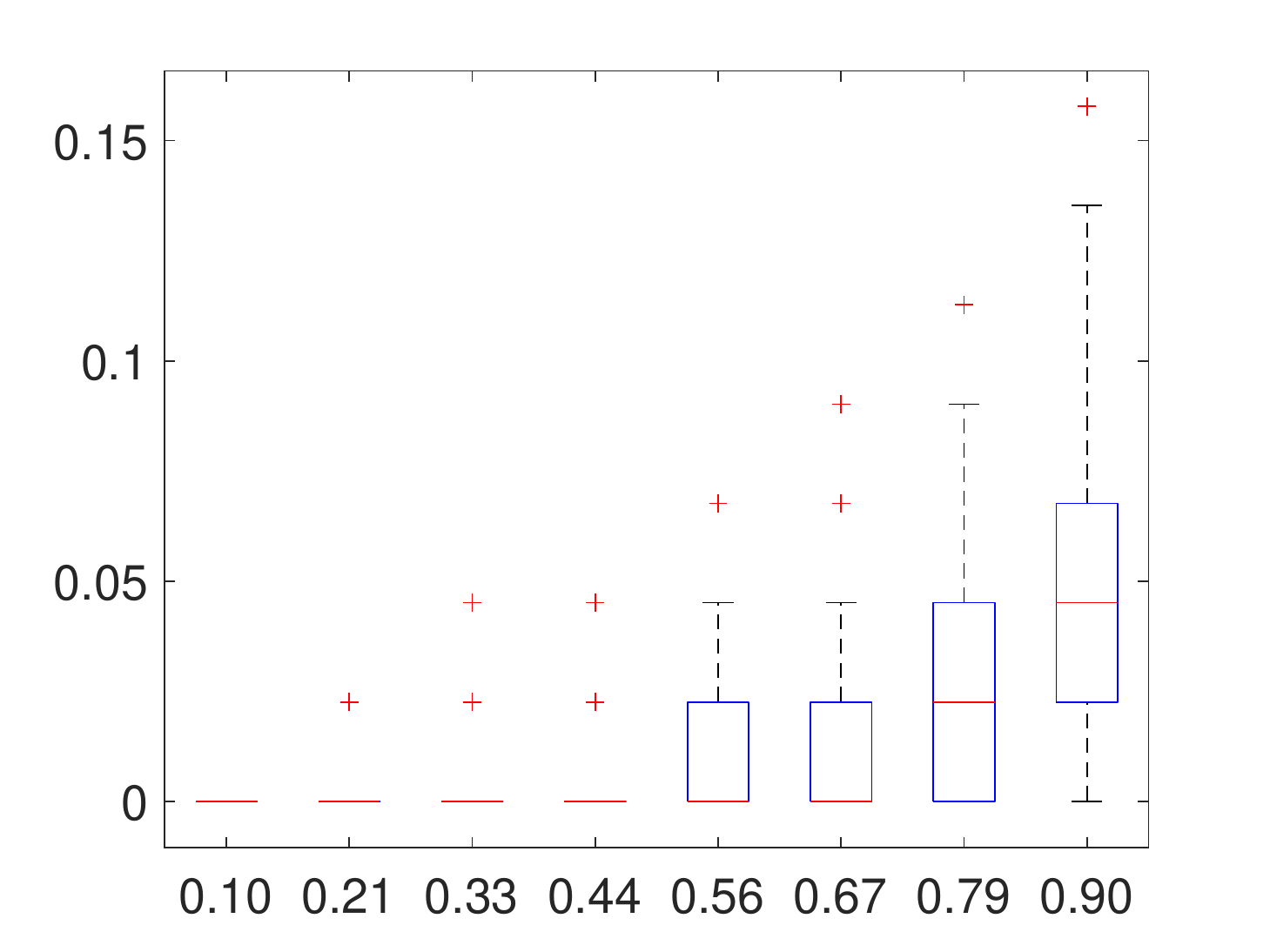}}
\caption{False detection rate ($\%$)}
 \end{figure}
 
\subsubsection{$n=200$}

\begin{figure}[H]
\setcounter{subfigure}{0}
\centering
\subfloat[RS]{\includegraphics[width=0.27\linewidth]{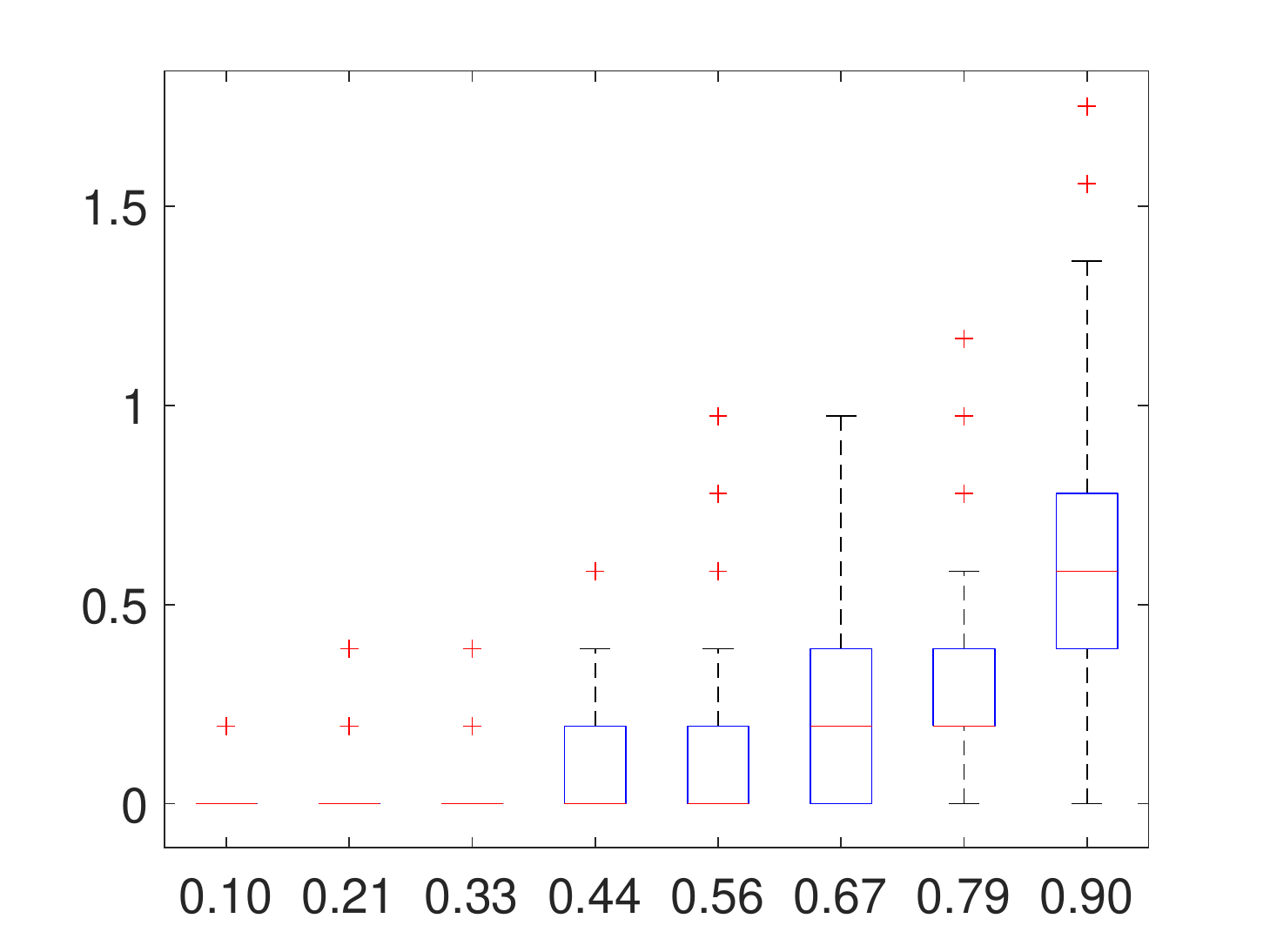}}
\subfloat[CV]{\includegraphics[width=0.27\linewidth]{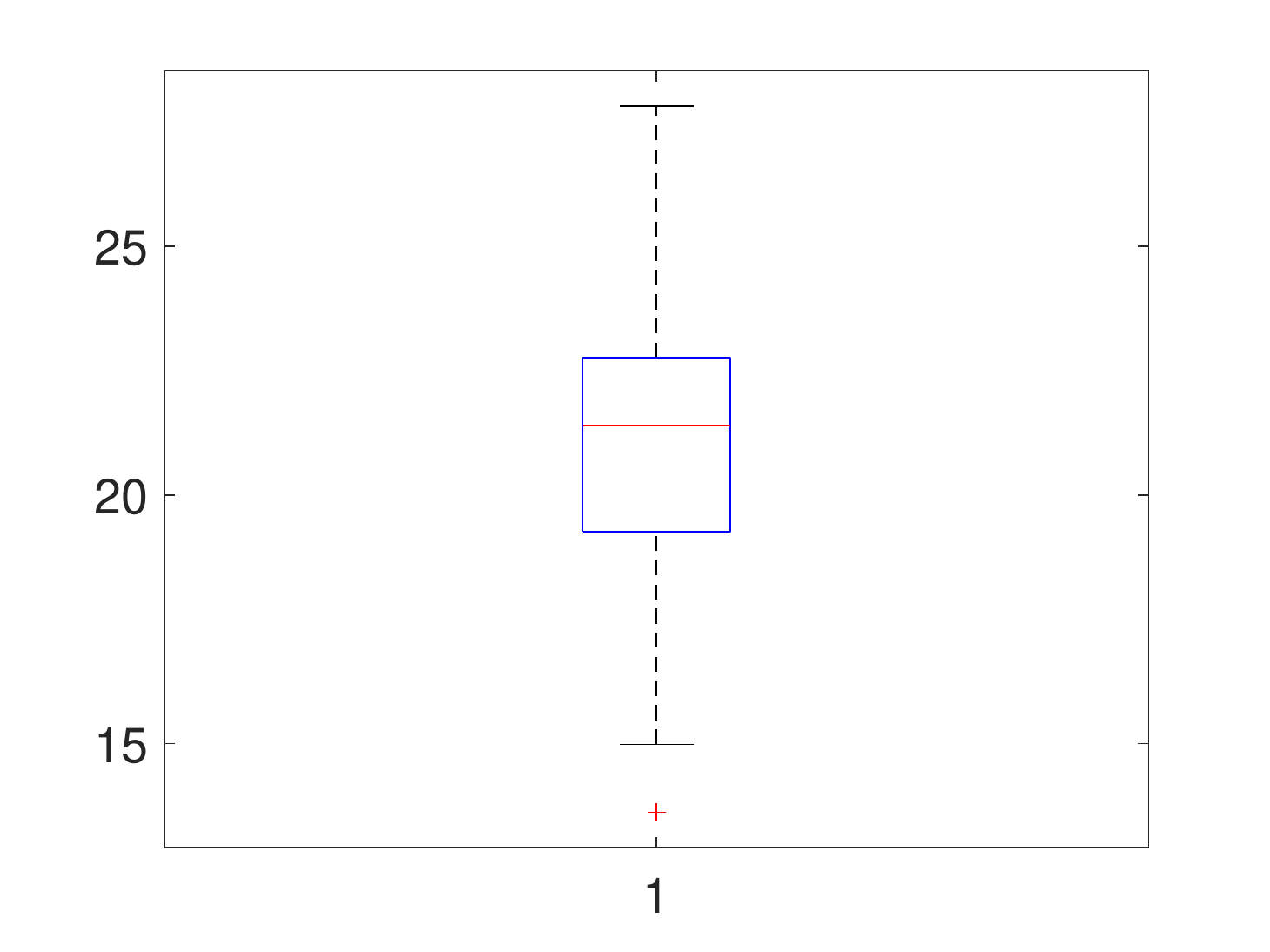}}
\subfloat[RWP]{\includegraphics[width=0.27\linewidth]{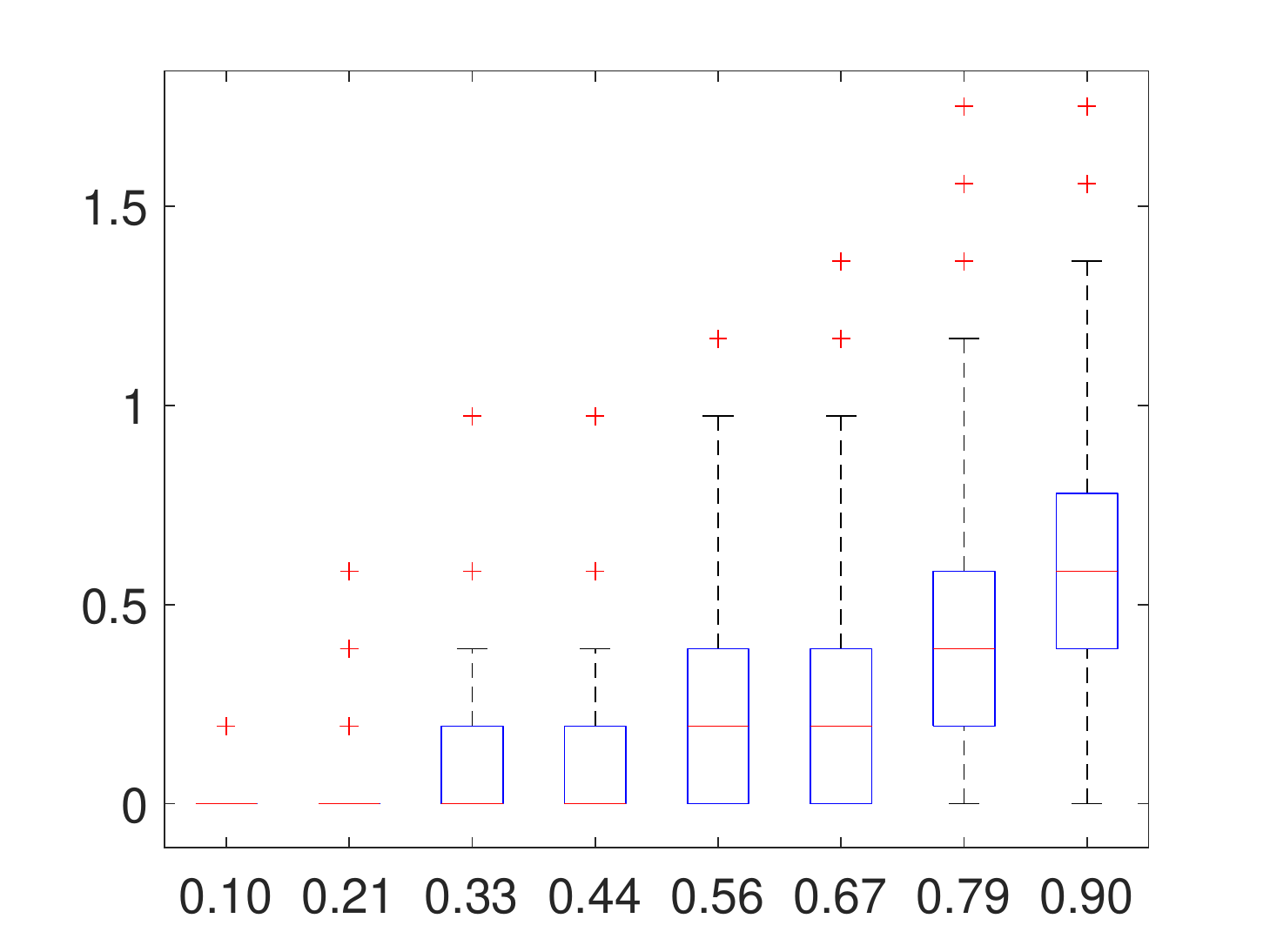}}
\caption{True positive rate ($\%$)}
 \end{figure}

\begin{figure}[H]
\setcounter{subfigure}{0}
\centering
\subfloat[RS]{\includegraphics[width=0.27\linewidth]{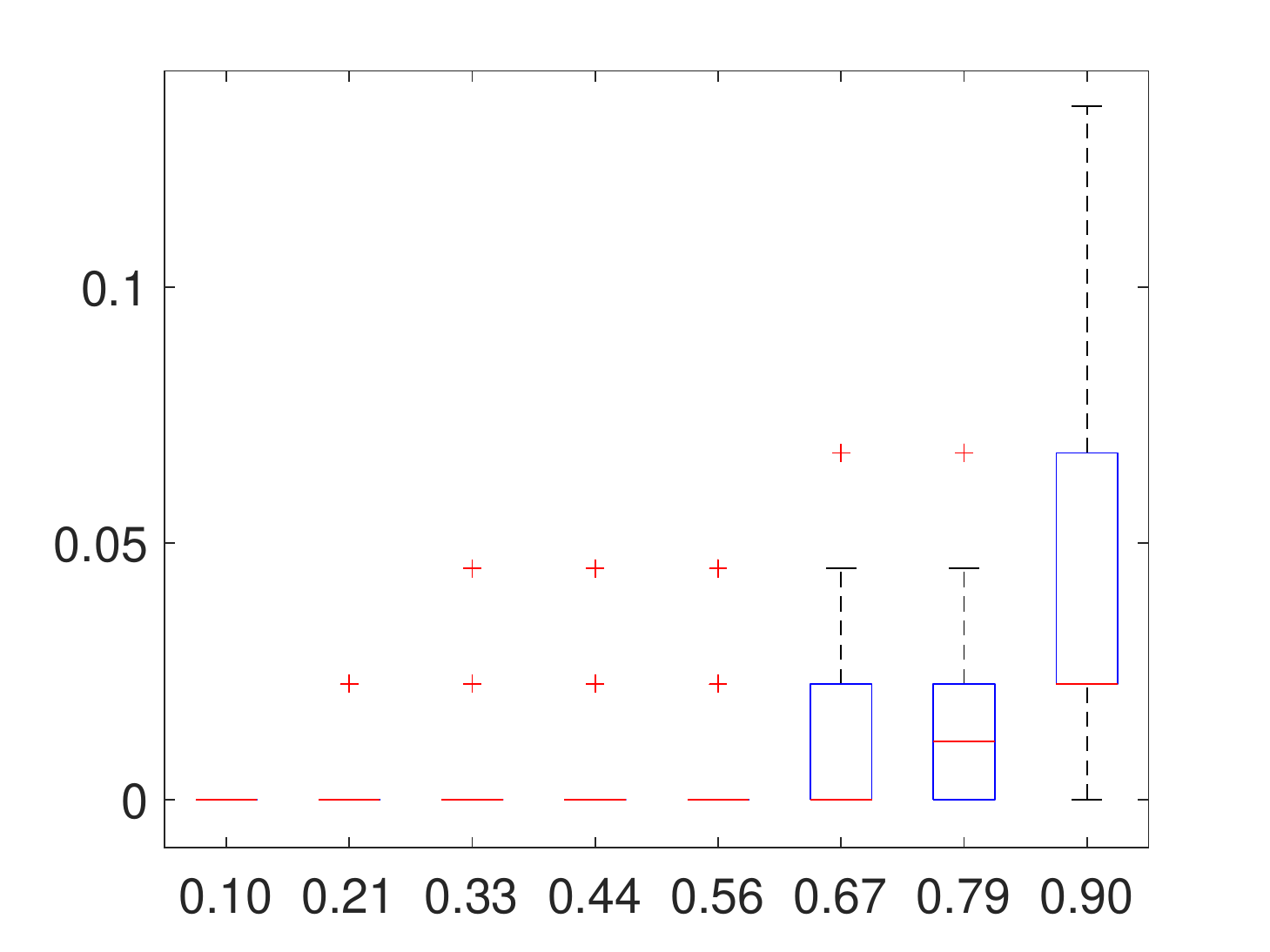}}
\subfloat[CV]{\includegraphics[width=0.27\linewidth]{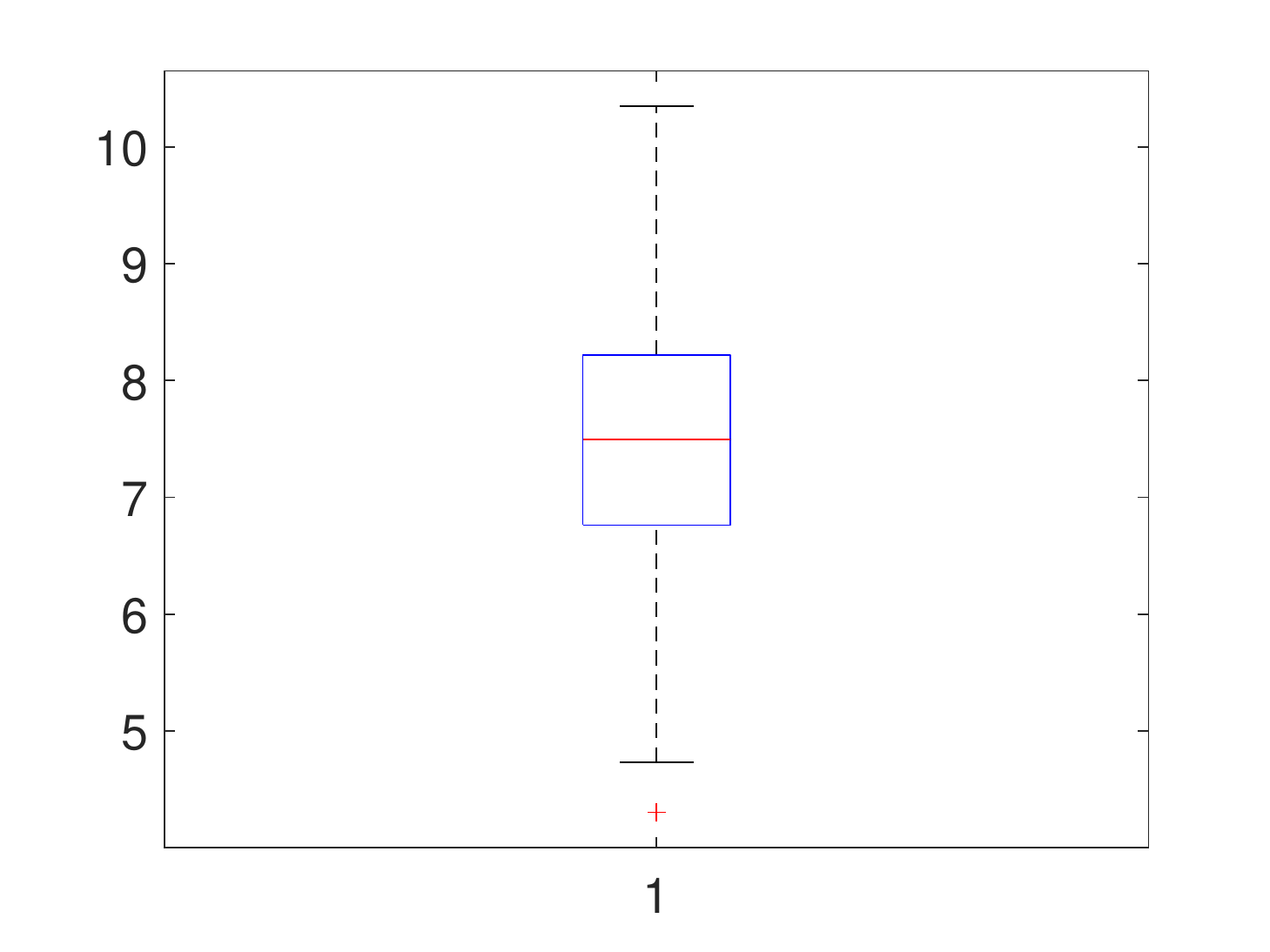}}
\subfloat[RWP]{\includegraphics[width=0.27\linewidth]{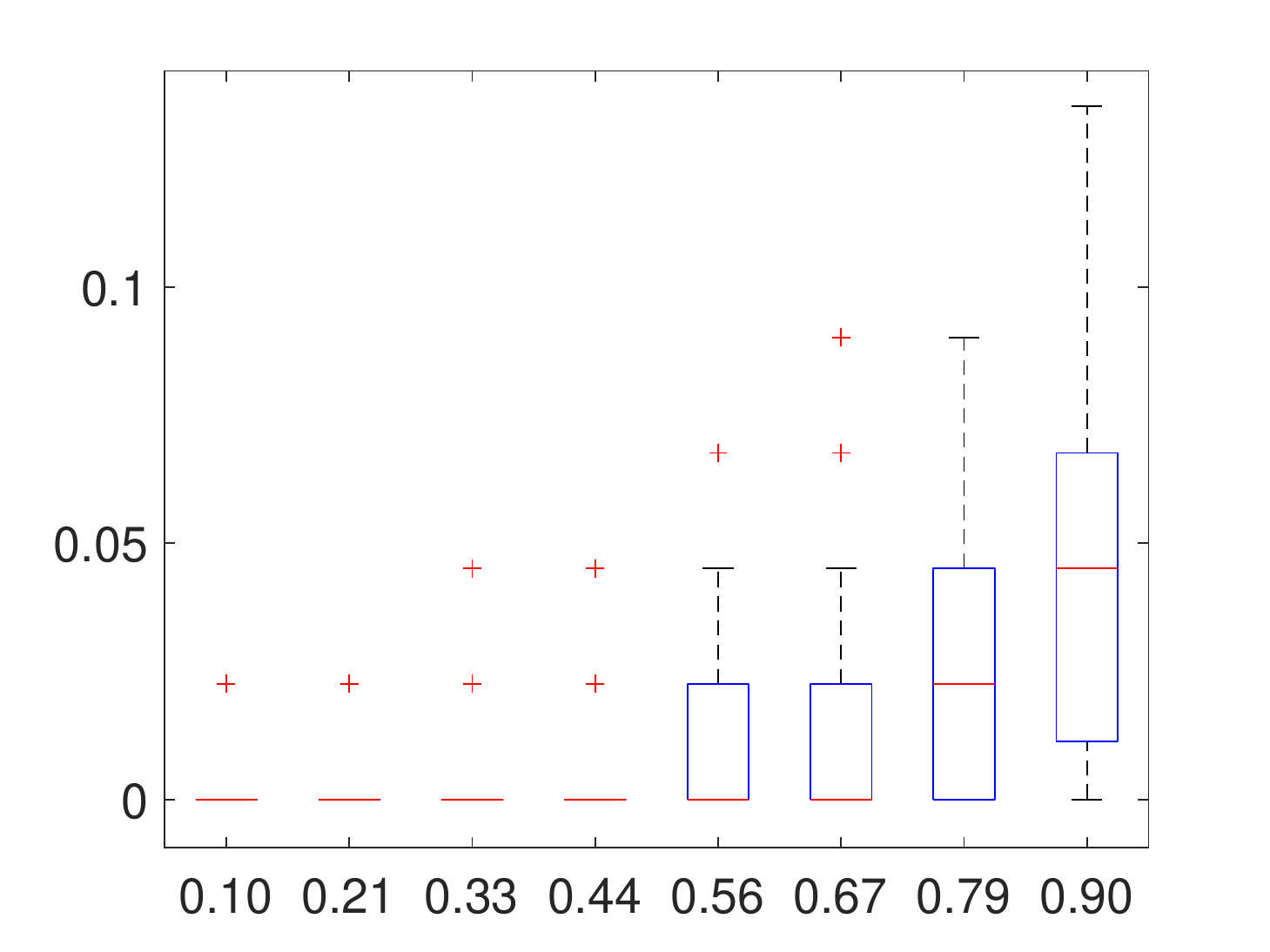}}
\caption{False detection rate ($\%$)}
 \end{figure}

\subsubsection{$n=1000$}

\begin{figure}[H]
\setcounter{subfigure}{0}
\centering
\subfloat[RS]{\includegraphics[width=0.27\linewidth]{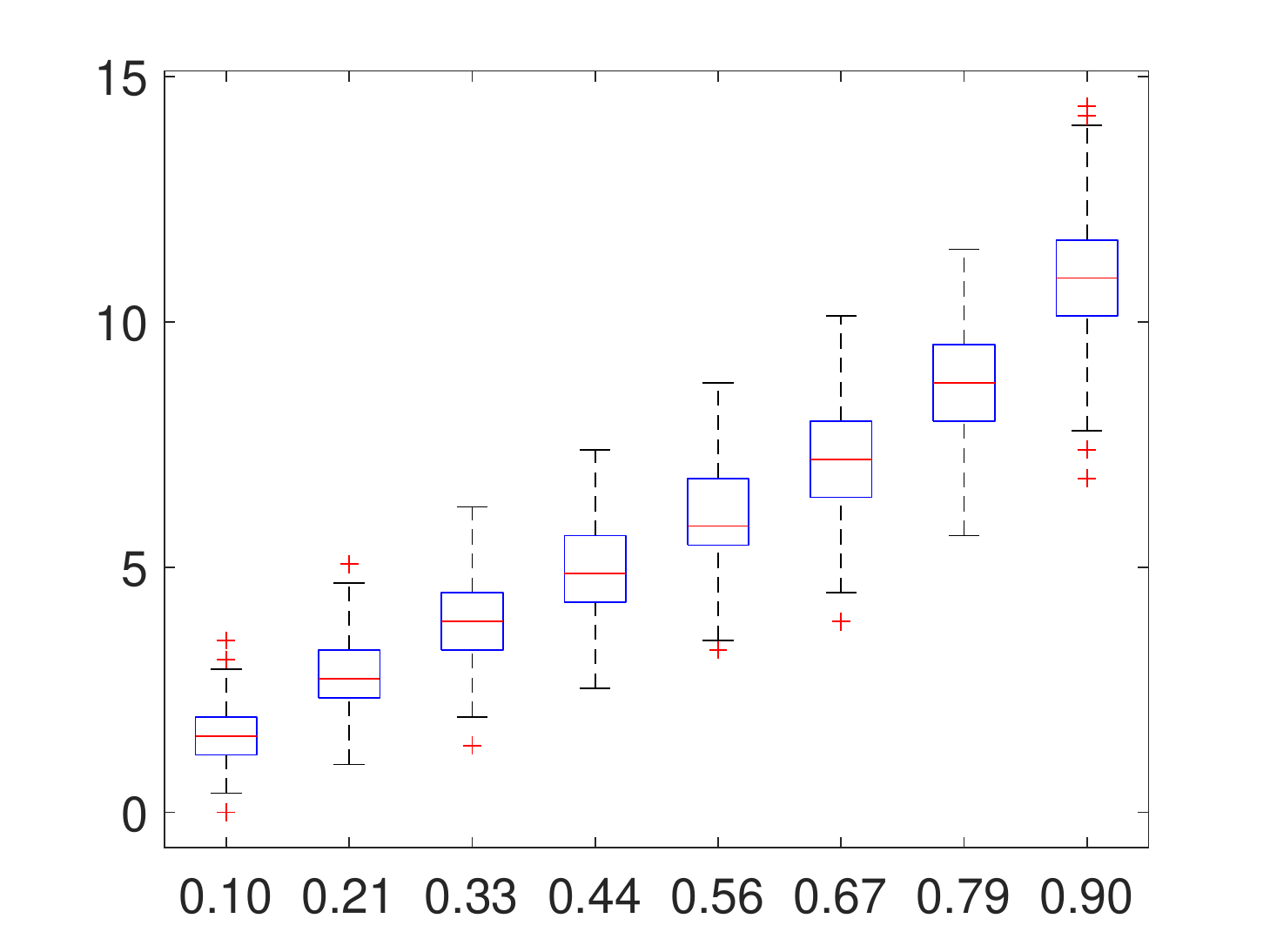}}
\subfloat[CV]{\includegraphics[width=0.27\linewidth]{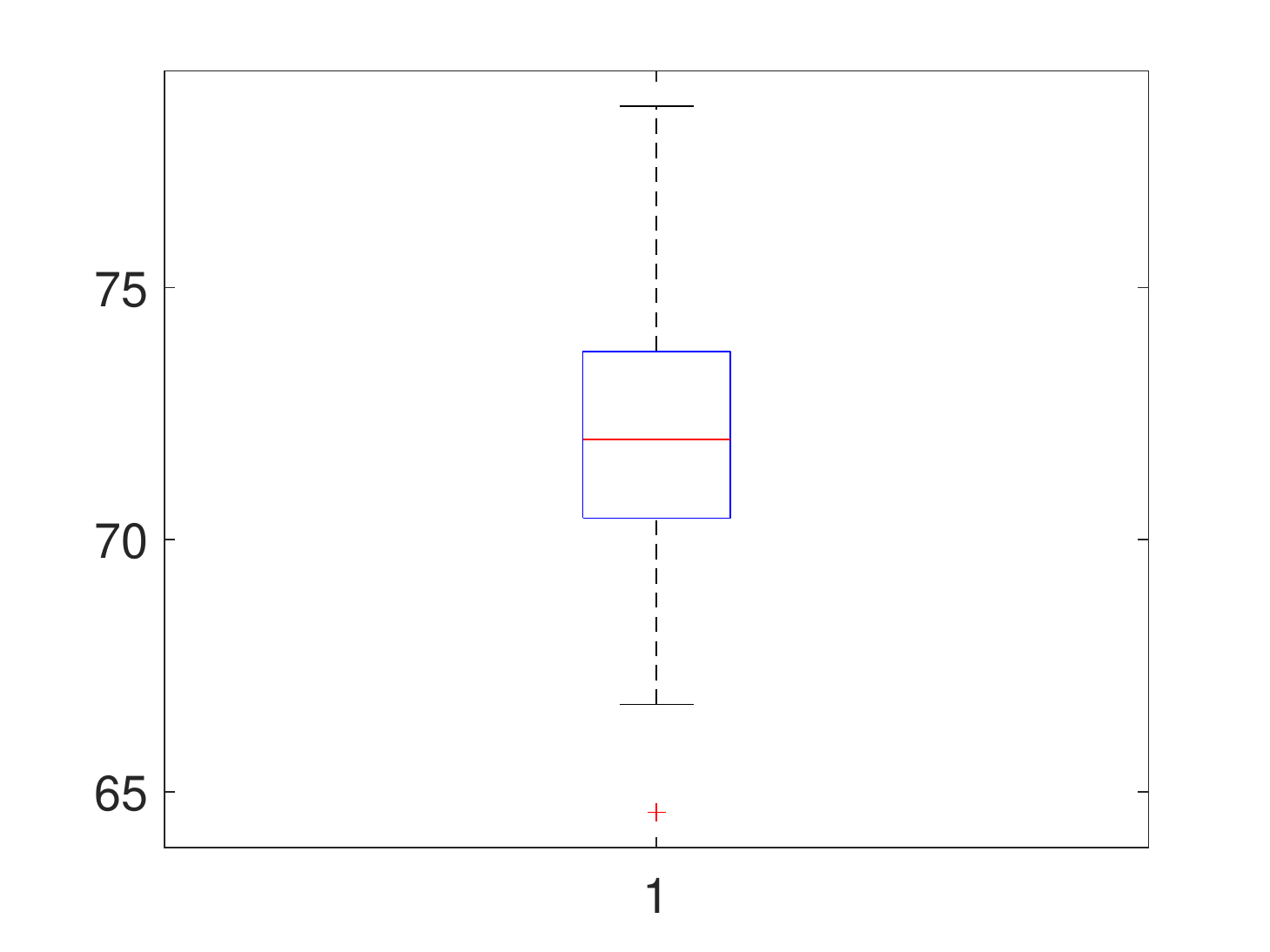}}
\subfloat[RWP]{\includegraphics[width=0.27\linewidth]{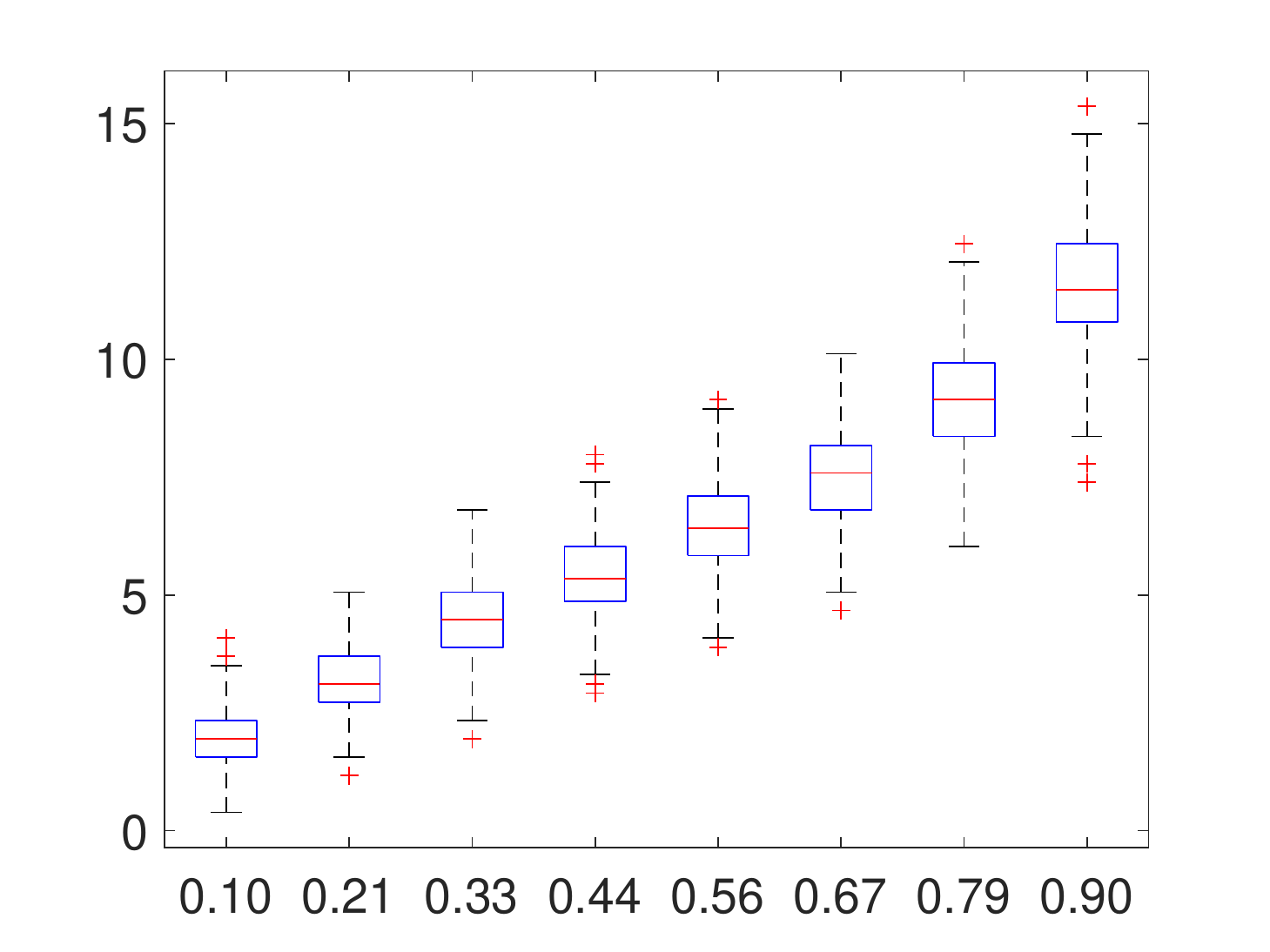}}
\caption{True positive rate ($\%$)}
 \end{figure}
%

\begin{figure}[H]
\setcounter{subfigure}{0}
\centering
\subfloat[RS]{\includegraphics[width=0.27\linewidth]{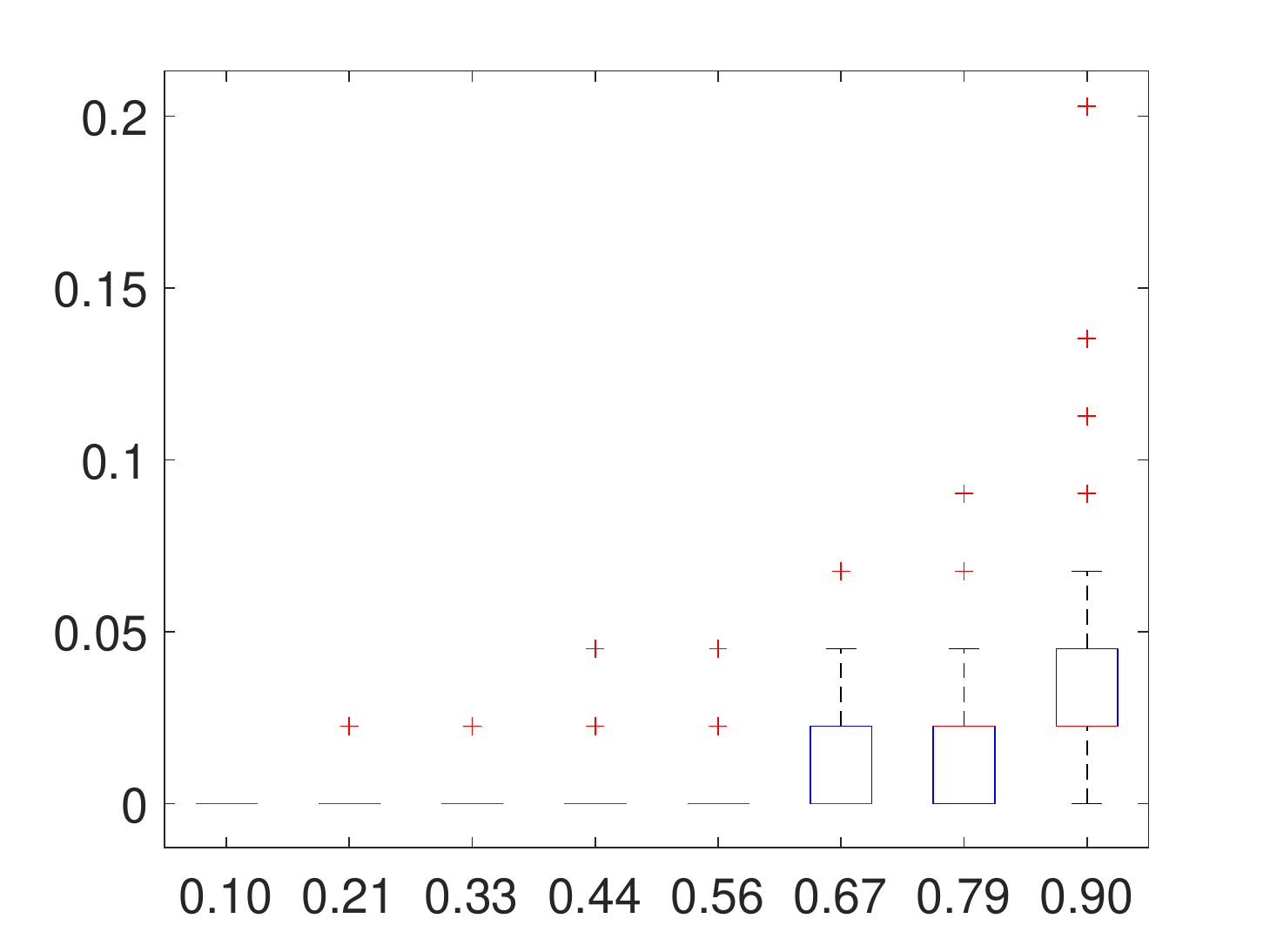}}
\subfloat[CV]{\includegraphics[width=0.27\linewidth]{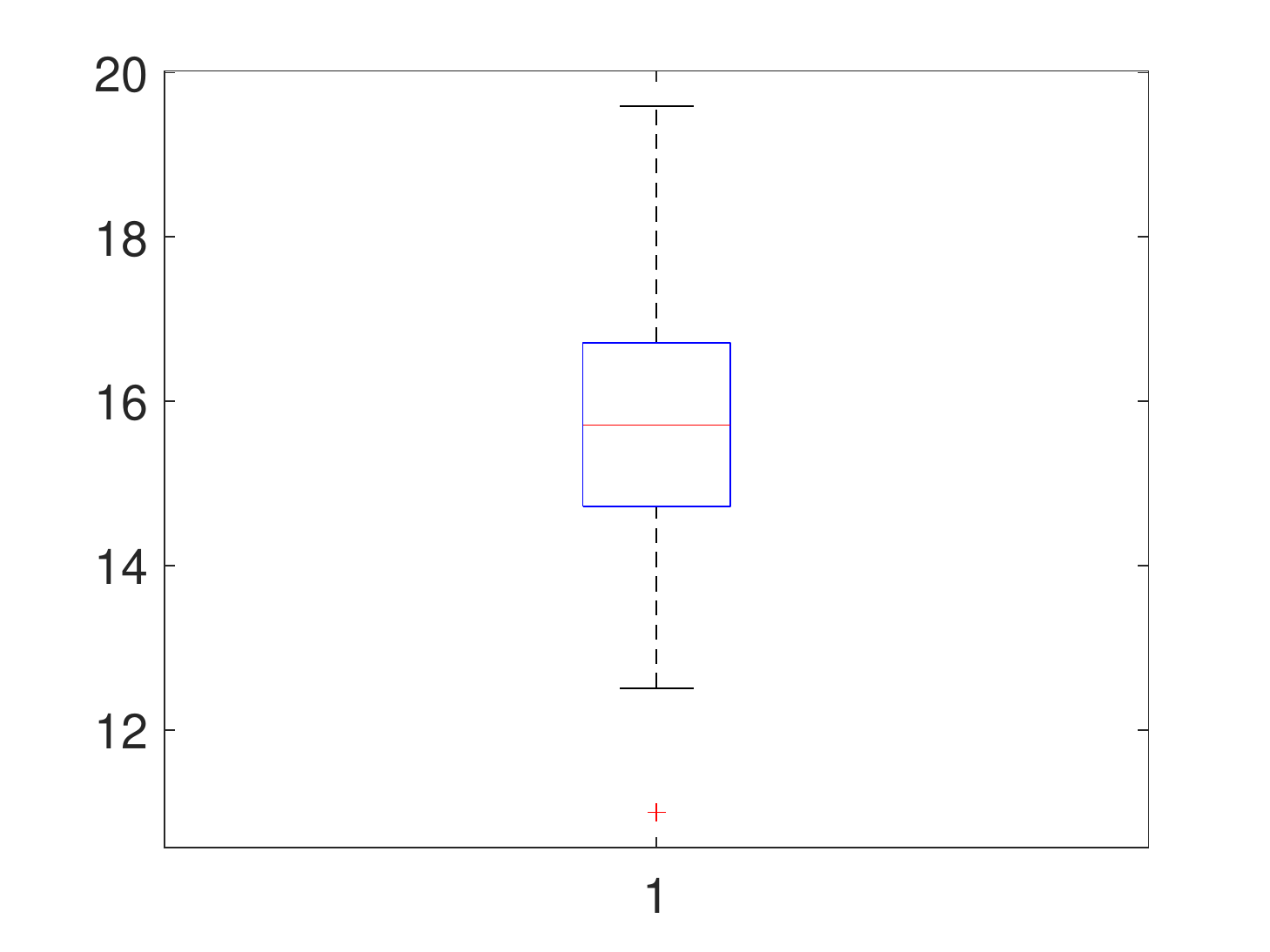}}
\subfloat[RWP]{\includegraphics[width=0.27\linewidth]{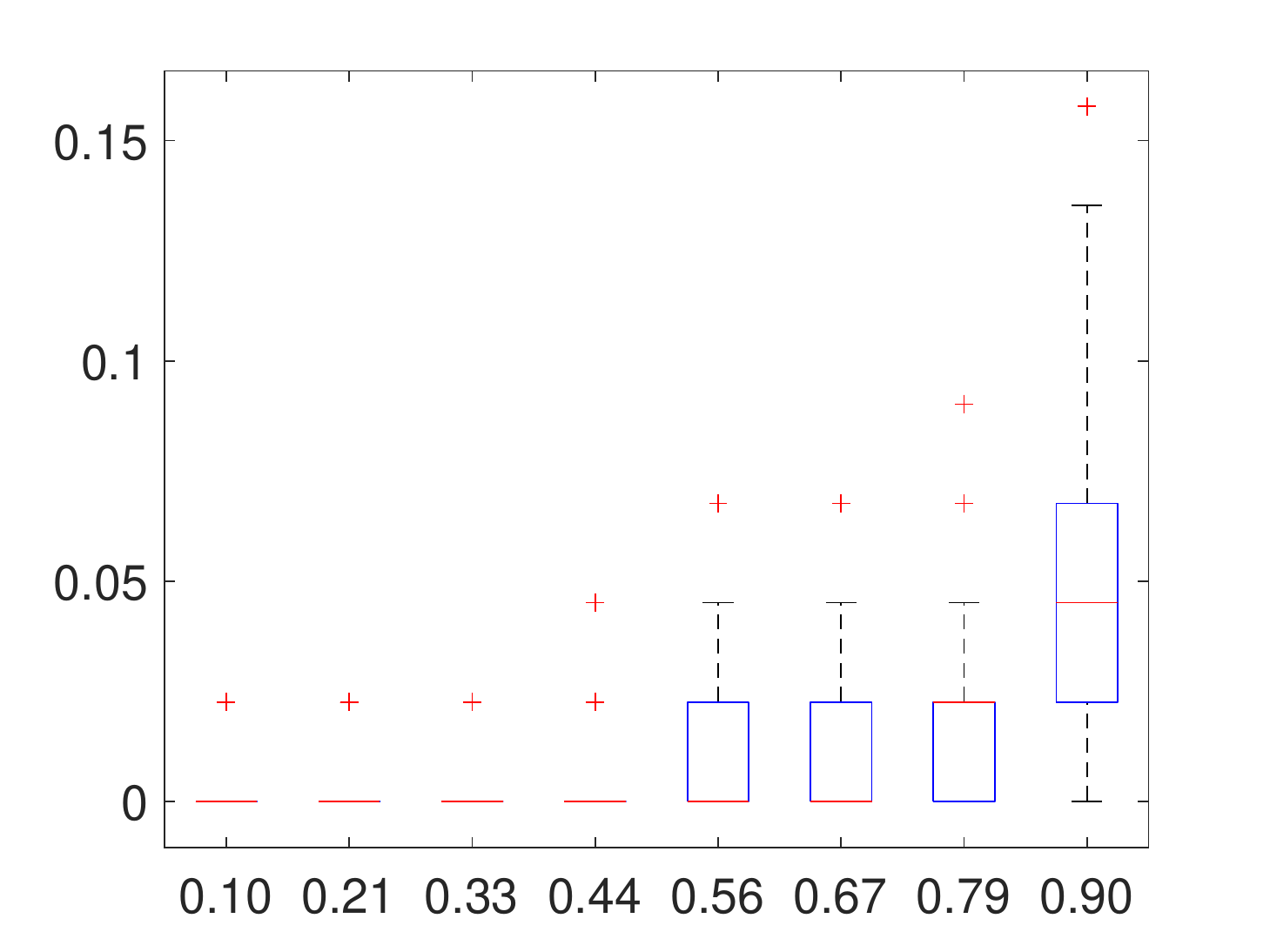}}
\caption{False detection rate ($\%$)}
 \end{figure}

\end{document}